 \documentclass[final,review,3pt,times]{elsarticle}

 \usepackage{etex}

\usepackage{amssymb}
\usepackage{amsmath}
\usepackage{amsfonts}
\usepackage{mathtools}
\usepackage{bm}
\usepackage{tikz}
\usetikzlibrary{trees,shapes}
\usepackage{pifont}
\usepackage{pgfplots}
\usepackage{subcaption}
\usepackage[all]{xy}

\usepackage{amsthm}

\newdefinition{example}{Example}
\newtheorem{theorem}{Theorem}
\newtheorem{lemma}{Lemma}
\newtheorem{proposition}{Proposition}
\newdefinition{definition}{Definition}
\newtheorem{corollary}{Corollary}

\newcommand{\cmark}{\ding{51}}%
\newcommand{\xmark}{\ding{55}}%
\newcommand{\T}{\mathbb{T}}

\newcommand{\G}{\mathcal{G}}
\newcommand{\D}{\textnormal{CD}}
\newcommand{\CD}{\textnormal{CD}}
\newcommand\independent{\protect\mathpalette{\protect\independenT}{\perp}}
\def\independenT#1#2{\mathrel{\rlap{$#1#2$}\mkern2mu{#1#2}}}
\usepackage{verbatim}

\usepackage{sansmath}
\usepackage{hyperref}
\usepackage{natbib}

\begin{document}

\begin{frontmatter}



\title{Sensitivity analysis, multilinearity and beyond}


\author{Manuele Leonelli}

\address{Departamento de Estat\'{i}stica, Universidade Federal do Rio de Janeiro, Rio de Janeiro, Brazil}
\author{Christiane G\"{o}rgen and Jim Q. Smith}

\address{Department of Statistics, The University of Warwick, Coventry, UK}

\begin{abstract}
Sensitivity methods for the analysis of the outputs of discrete Bayesian networks have been extensively studied and implemented in different software packages. These methods usually focus on the study of sensitivity functions and on the impact of a parameter change to the Chan-Darwiche distance. Although not fully recognized, the majority of these results  rely  heavily on the multilinear structure of atomic probabilities in terms of the conditional probability parameters associated with this type of network.  By defining a statistical model through the polynomial expression of its associated defining conditional probabilities, we develop here a unifying approach to sensitivity methods applicable to a large suite of models including extensions of Bayesian networks, for instance context-specific and dynamic ones.  Our algebraic approach enables us to prove that for models whose defining polynomial is multilinear both the Chan-Darwiche distance and any divergence in the family of $\phi$-divergences are minimized for a certain class of multi-parameter contemporaneous variations when parameters are proportionally covaried.  
\end{abstract}

\begin{keyword}
 Bayesian networks \sep CD distance  \sep Interpolating Polynomial \sep Sensitivity Analysis \sep $\phi$-divergences. 



\end{keyword}

\end{frontmatter}


\section{Introduction}
\label{sec:intro}
Many discrete statistical problems in a variety of domains are nowadays often modeled using \textit{Bayesian networks} (BNs) \cite{Pearl1988}. There are now thousands of practical applications of these models \cite{Aguilera2011, Cano2004, Heckerman1995, Jordan2004}, which have spawned many useful technical developments: including a variety of fast exact, approximate and symbolic propagation algorithms for the computation of probabilities that exploit the underlying graph structure \cite{Cowell2007,Dagum1993, Darwiche2003}. Some of these advances have been hard-wired into software \cite{Chan2002,Korb2010,Low2012} which has further increased the applicability and success of these methods.

However, BN modeling would not have experienced such a widespread application without tailored methodologies of \textit{model validation}, i.e. checking that a model produces outputs that are in line with current understanding, following a defensible and expected mechanism \cite{French2003, Pitchforth2013}. Such techniques are now well established for BN models \cite{Chen2012,Korb2010,Pitchforth2013,Pollino2007}. These are especially fundamental for expert elicited models, where both the probabilities and the covariance structure are defined from the suggestions of domains experts, following knowledge engineering protocols tailored to the BN's bulding process \cite{Neil2000,Rajabally2004}. We can broadly break down the validation process into two steps: the first concerns the auditing of the underlying graphical structure; the second, assuming the graph represents a user's beliefs, checks the impact of the numerical elicited probabilities within this parametric family on outputs of interest. The focus of this paper lies in this second validation phase, usually called a \textit{sensitivity analysis}.    

The most common investigation is the so-called \textit{one-way} sensitivity analysis, where the impacts of changes made to a single probability parameter are studied. Analyses where more than one parameter at a time is varied are usually referred to as \textit{multi-way}. In both cases a complete sensitivity analysis for discrete BNs often involves the study of \textit{Chan-Darwiche (CD) distances} \cite{Chan2002,Chan2004,Chan2005} and \textit{sensitivity functions} \cite{Coupe2002,Gaag2007}. The CD distance is used to quantify  global changes. It measures how the overall distribution behaves when one (or more) parameter is varied. A significant proportion of research has focused on identifying parameter changes such that the original and the \lq{v}aried\rq{} BN distributions are close in CD distance \cite{Chan2005,Renooij2014}. This is minimized when, after a single arbitrary parameter change, other covarying parameters, e.g. those from the same conditional distribution, have the same proportion of the residual probability mass as they originally had. Sensitivity functions, on the other hand, model local changes with respect to an output of interest. These describe how that output probability varies as one (or potentially more) parameter is allowed to be changed. Although both these concepts can be applied to generic Bayesian analyses, they have almost exclusively been discussed and applied only within the BN literature (see \cite{Chan2005a,Charitos2006,Charitos2006a,Renooij2012} for some exceptions). This is because the computations of both CD distances and sensitivity functions are particularly straightforward for BN models. 

In this paper we introduce a unifying comprehensive framework for certain multi-way analyses, usually called in the context of BNs \textit{single full conditional probability table (CPT) analyses} - where one parameter from each CPT of one vertex of a BN given each configurations of its parents is varied. Using the notion of an interpolating polynomial \cite{Pistone2001} we are able to describe a large variety of models based on their polynomial form. Then, given this algebraic carachterization, we demonstrate that one-way sensitivity methods defined for BNs can be generalized to single full CPT analyses for any model whose interpolating polynomial is multilinear, for example context-specific BNs \cite{Boutilier1996} and chain event graphs \cite{Smith2008}. Because of both the lack of theoretical results justifying their use and the increase in computational complexity, multi-way methods have not been extensively discussed in the literature: see \cite{Bolt2015,Chan2004,Gomez2013} for some exceptions. This paper aims at providing a comprehensive theoretical toolbox to start applying such analyses in practice.

Importantly, our polynomial approach enables us to prove that single full CPT analyses in any multilinear polynomial model are optimal under proportional covariation in the sense that the CD distance between the original and the varied distributions is minimized. The optimality of this covariation method has been an open problem in the sensitivity analysis literature for quite some time \cite{Chan2004,Renooij2014}. However, we are able to provide further theoretical justifications for the use of proportional covariation in single full CPT analyses. We demonstrate below that for any multilinear model this scheme minimizes not only the CD distance, but also any divergence in the family of $\phi$-divergences \cite{Ali1966,Csiszar1963}. The class of $\phi$-divergences include a very large number of divergences and distances (see e.g. \cite{Pardo2005} for a review), including the famous Kullback-Leibler (KL) divergence \cite{Kullback1951}. The application of KL distances in sensitivity analyses of BNs has been almost exclusively restricted to the case when the underlying distribution is assumed Gaussian \cite{Gomez2007,Gomez2013}, because in discrete BNs the computation of such a divergence requires more computational power than for CD distances. We will demonstrate below that this additional complexity is a feature shared by any divergence in the family of $\phi$-divergences.      

However, by studying sensitivity analysis from a polynomial point of view,  we are able to consider a much larger class of models for which such methods are very limited. We investigate the properties of one-way sensitivity analysis in models whose interpolating polynomial is not multilinear, which are usually associated to dynamic settings where probabilities are recursively defined. This difference gives us an even richer class of sensitivity functions as shown in \cite{Charitos2006,Charitos2006a,Renooij2012} for certain dynamic BN models, which are not simply linear but more generally polynomial. We further introduce a procedure to compute the CD distance in these models and demonstrate that no unique updating of covarying parameters lead to the smallest CD distance between the original and the varied distribution.

The paper is structured as follows. In Section \ref{sec:int} we define interpolating polynomials and demonstrate that many commonly used models entertain a polynomial representation. In Section \ref{sec:div} we review a variety of divergence measures. Section \ref{sec:multi} presents a variety of results for single full CPT sensitivity analyses in multilinear models. In Section \ref{sec:pol} the focus moves to non-multilinear models and one-way analyses. We conclude with a discussion.  

\section{Multilinear and polynomial parametric models}
\label{sec:int}
In this section we first provide a generic definition of a parametric statistical model together with the notion of interpolating polynomial. We then categorize parametric models according to the form of their interpolating polynomial and show that many commonly used models fall within two classes.
\subsection{Parametric models and interpolating polynomials}
Let $\bm{Y}=(Y_1,\dots,Y_m)$ be a random vector with an associated discrete and finite sample space $\bm{\mathbb{Y}}$, with $\#\mathbb{Y}=n$. Although  our methods straightforwardly applies when the entries of $\bm{Y}$ are random vectors, for ease of notation, we henceforth assume its elements are univariate.

\begin{definition} 
Denote by $\boldsymbol{p}_{\theta}=(p_{\theta}(\bm{y})~|~\bm{y}\in\bm{\mathbb{Y}})$ the vector of values of a probability mass function $p_{\theta}:\bm{\mathbb{Y}}\to[0,1]$ which depends on a choice of parameters $\theta\in\mathbb{R}^k$. The entries of $\bm{p}_\theta$ are called \emph{atomic probabilities} and the elements $\bm{y}\in\bm{\mathbb{Y}}$ \emph{atoms}.
\end{definition}

\begin{definition}
A discrete \emph{parametric statistical model} on $n\in\mathbb{N}$ atoms is a subset $\mathbb{P}_{\Psi}\subseteq\Delta_{n-1}$ of the $n-1$ dimensional probability simplex, where
\begin{equation}\label{eq:psi}
\Psi~:\quad\mathbb{R}^k\to\mathbb{P}_{\Psi},~ \theta\mapsto \boldsymbol{p}_{\theta},
\end{equation}
is a bijective map identifying a particular choice of parameters $\theta\in\mathbb{R}^k$ with one vector of atomic probabilities. The map $\Psi$ is called a \emph{parametrisation} of the model.
\label{def:par}
\end{definition}

The above definition is often encountered in the field of \textit{algebraic statistics}, where properties of statistical models are studied using techniques from algebraic geometry and commutative computer algebra, among others \cite{Drton2009,Riccomagno2009}. We next follow \cite{Gorgen2015b} in extending some standard terminology.

\begin{definition}
\label{def:mon}
A model $\mathbb{P}_{\Psi}\subseteq\Delta_{n-1}$ has a \emph{monomial parametrisation} if
\[
p_{\theta}(\bm{y})=\boldsymbol{\theta}^{\bm{\alpha}_{\bm{y}}},\quad\text{for all }\bm{y}\in\bm{\mathbb{Y}},
\]
where $\bm{\alpha}_{\bm{y}}\in\mathbb{N}_{0}^k$ denotes a vector of exponents and $\boldsymbol{\theta}^{\bm{\alpha}_{\bm{y}}}=\theta_{1}^{\alpha_{1,\bm{y}}}\cdots \theta_{k}^{\alpha_{k,\bm{y}}}$ is a monomial. Then equation (\ref{eq:psi}) is a monomial map and $\boldsymbol{\theta}^{\bm{\alpha}_{\bm{y}}}\in\mathbb{R}_k[\Theta]$, for all $\bm{y}\in\mathbb{Y}$. Here $\Theta=\{\theta_{1},\ldots,\theta_{k}\}$ is the set of indeterminates and $\mathbb{R}_k[\Theta]$ is the polynomial ring over the field $\mathbb{R}$.
\end{definition}

For models entertaining a monomial parametrisation the network polynomial we introduce in Definition \ref{def:netw} below concisely captures the model structure and provides a platform to answer inferential queries \cite{Darwiche2003, Gorgen2015a}.
 
 \begin{definition}
  \label{def:netw}
 The \emph{network polynomial} of a model $\mathbb{P}_{\Psi}$ with monomial parametrisation $\Psi$ is given by
 \begin{equation*}
 c_{\mathbb{P}_\Psi}(\theta,\lambda)=\sum_{\bm{y}\in\bm{\mathbb{Y}}}\lambda_{\bm{y}}\bm{\theta}^{\bm{\alpha}_{\bm{y}}},
 \end{equation*}
 where $\lambda_{\bm{y}}$ is an indicator function for the atom $\bm{y}$.
 \end{definition}
Probabilities of events in the underlying sigma-field can be computed from the network polynomial by setting equal to one the indicator function of atoms associated to that event. In the following it will be convenient to work with a special case of the network polynomial where all the indicator functions are set to one.
\begin{definition} 
The \emph{interpolating polynomial} of a model $\mathbb{P}_{\Psi}$ with monomial parametrisation $\Psi$ is given by the sum of all atomic probabilities,
\[
c_{\mathbb{P}_{\Psi}}(\theta)=\sum_{\bm{\alpha}\in\bm{\mathbb{A}}}\bm{\theta}^{\bm{\alpha}},
\]
where $\mathbb{A}=\{\bm{\alpha}_{\bm{y}}\ | \ \bm{y}\in\mathbb{Y}\}\subset \mathbb{N}^k_0$.
\end{definition}

\subsection{Multilinear models}
In this work we will mostly focus on parametric models whose interpolating polynomial is multilinear.
\begin{definition}
We say that a parametric model $\mathbb{P}_{\Psi}$ is \emph{multilinear} if its associated interpolating polynomial is multilinear, i.e. if $\mathbb{A}\subseteq\{0,1\}^k$.\label{def:multi}
\end{definition}

We note here that a great portion of well-known non-dynamic graphical models are multilinear. We explicitly show below that this is the case for BNs and context-specific BNs \cite{Boutilier1996}. In \cite{Gorgen2015a} we showed that certain chain event graph models \cite{Smith2008} have multilinear interpolating polynomial.  In addition, decomposable undirected graphs and probabilistic chain graphs \cite{Lauritzen1996} can be defined to have a monomial parametrisation whose associated interpolating polynomial is multilinear. An example of models non entertaining a monomial parametrisation in terms of atomic probabilities are non-decomposable undirected graphs, since their joint distribution can be written as a rational function of multilinear functions \cite{Chan2005a}.  
 
\subsubsection{Bayesian networks}
For an $m\in\mathbb{N}$, let $[m]=\{1,\dots,m\}$. We denote with $Y_i$, $i\in[m]$, a generic discrete random variable and with $\mathbb{Y}_i=\{0,\dots,m_i\}$ its associated sample space. For an $A\subseteq[m]$, we let $\bm{Y}_A=(Y_i)_{i\in A}$ and $\bm{\mathbb{Y}}_A=\times_{i\in A}\mathbb{Y}_i$. Recall that for three random vectors $\bm{Y}_i$, $\bm{Y}_j$ and $\bm{Y}_k$, we say that $\bm{Y}_i$ is conditional independent of $\bm{Y}_j$ given $\bm{Y}_k$, and write $\bm{Y}_i\independent \bm{Y}_j\;|\; \bm{Y}_k$, if $\Pr(\bm{Y}_i=\bm{y}_i\;|\;\bm{Y}_j=\bm{y}_j,\bm{Y}_k=\bm{Y}_k)=\Pr(\bm{Y}_i=\bm{y}_i\;|\;\bm{Y}_k=\bm{Y}_k)$.
\begin{definition}
A BN over a discrete random vector $\bm{Y}_{[m]}$ consists of 
\begin{itemize}
\item $m-1$ \textit{conditional independence} statements of the form $Y_i\independent \bm{Y}_{[i-1]}\;|\, \bm{Y}_{\Pi_i}$, where $\Pi_i\subseteq [i-1]$;
\item a \textit{directed acyclic graph} (DAG) $\mathcal{G}$ with vertex set $V(\G)=\{Y_i: i \in [m]\}$ and edge set $E(\G)=\{(Y_i,Y_j):j\in[m],i\in\Pi_j\}$;
\item conditional probabilities  $\theta_{i_j\bm{\pi}}=\Pr(Y_i=j\;|\;\bm{Y}_{\Pi_i}=\bm{\pi})$ for every $j\in\mathbb{Y}_i$, $\bm{\pi}\in\bm{\mathbb{Y}}_{\Pi_i}$ and $i\in[m]$. 
\end{itemize}
\end{definition}
The vector $\bm{Y}_{\Pi_i}$, $i\in[m]$, includes the \textit{parents} of the vertex $Y_i$, i.e. those vertices $Y_j$ such that there is an edge $(Y_j,Y_i)$ in the DAG $\G$ of the BN.
 
From \cite{Chan2004} we know that for any atom $\bm{y}\in\bm{\mathbb{Y}}_{[m]}$ its associated monomial in the network polynomial can be written as
\[
p_\theta(\bm{y})=\prod_{\bm{y}\sim \{i_j,\bm{\pi}\}}\lambda_{i_j}\theta_{i_j\bm{\pi}},
\]
where $\sim$ denotes the compatibility relation among instantiations. 

\begin{lemma}
From \cite{Darwiche2003, Gorgen2015a}, the interpolating polynomial of a BN model can be written as
\begin{equation}
\label{eq:BNpol}
c_{\textnormal{BN}}(\theta)=\sum_{\bm{y}\in\bm{\mathbb{Y}}_{[m]}}\prod_{\bm{y}\sim \{i_j,\bm{\pi}\}}\theta_{i_j\bm{\pi}}.
\end{equation}
\end{lemma}
From Equation (\ref{eq:BNpol}) we can immediately deduce the following.
\begin{proposition}
A BN is a multilinear parametric model, whose interpolating polynomial is homogeneous with monomials of degree $m$.
\end{proposition}

\begin{example}
\label{ex:BN}
Suppose a newborn is at risk of acquiring a disease and her parents are offered a screening test ($Y_1$) which can be either positive ($Y_1=1$) or negative ($Y_1=0$). Given that the newborn can either severely ($Y_2=2$) or mildly ($Y_2=1$) contract the disease or remain healthy ($Y_2=0$), her parents can then decide whether or not to give her a vaccine to prevent a relapse ($Y_3 = 1$ and $Y_3 = 0$, respectively). We assume that the parents' decision about the vaccine does not depend on the screening test if the newborn contracted the disease, and that the probability of being severely or mildly affected by the disease is equal for negative screening tests.
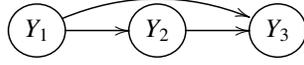
\begin{figure}
\entrymodifiers={++[o][F-]}
\centerline{
\xymatrix{
Y_1\ar[r]\ar@/^1pc/[rr]&Y_2\ar[r]&Y_3
}
}
\caption{A BN model for the medical problem in Example \ref{ex:BN}. \label{fig:BN}}
\end{figure}

The above situation can be described, with some loss of information, by the BN in Figure \ref{fig:BN}, with probabilities, for $i,k\in\{0,1\}$ and $j\in\{0,1,2\}$, 
\[
\Pr(Y_1=i)=\theta_{1_i},\quad \Pr(Y_2=j\,|\,Y_1=i)=\theta_{2_j1_i},\quad \Pr(Y_3=k\,|\,Y_2=j,Y_1=i)=\theta_{3_k2_j1_i}.
\]
 Its associated interpolating polynomial has degree $3$ and equals
\[
c_{\textnormal{BN}}(\theta)=\sum_{i=0}^1\sum_{j=0}^2\sum_{k=0}^1\theta_{1_i}\theta_{2_j1_i}\theta_{3_k2_j1_i}.
\]
\end{example}
\subsubsection{Context-specific Bayesian networks}
In practice it has been recognized that often conditional independence statements do not hold over the whole sample space of certain conditioning variables but only for a subset of this, usually referred to as a \textit{context}. A variety of methods have been introduced to embellish a BN with additional independence statements that hold only over contexts. A BN equipped with such embellishments is usually called \textit{context-specific BN}. Here we consider the representation known as \textit{ context specific independence (CSI)-trees} and introduced in \cite{Boutilier1996}. 

\begin{example}
\label{ex:CSBN}
Consider the medical problem in Example \ref{ex:BN}. Using the introduced notation, we notice that by assumption, for each $k=0,1$, the probabilities $\theta_{3_k2_11_i}$ are equal for all $i=0,1$ and called $\theta_{3_k2_11}$. Similarly, $\theta_{3_k2_21_i}$ are equal and called $\theta_{3_k2_21}$, $i,k=0,1$. Also $\theta_{2_21_0}=\theta_{2_11_0}$ are equal and called $\theta_{21_0}$. The first two constraints can be represented by the CSI-tree in Figure \ref{fig:CSI}, where the inner nodes are random variables and the leaves are entries of the CPTs of one vertex. The tree shows that, if $Y_2=1$ or $Y_2=2$ then no matter what the value of $Y_1$ is, the CPT for $Y_3=k$ will be equal to $\theta_{3_k2_11}$ and $\theta_{3_k2_11}$ respectively.
\begin{figure}
\centerline{
\xymatrix{
&&Y_2\ar[rdd]^{1}\ar[rrdd]^{2}\ar[ld]^0&\\
&Y_1\ar[dr]^{0}\ar[dl]^1&&\\
\theta_{3_k2_01_0}&&\theta_{3_k2_01_1}&\theta_{3_k2_11}&\theta_{3_k2_21}}}
\caption{CSI-tree associated to vertex $Y_3$ of the BN in Figure \ref{fig:BN} of Example \ref{ex:BN}, where $\theta_{3_k2_11}=\Pr(Y_3=k|Y_2=1)$, $\theta_{3_k2_01_1}=\Pr(Y_3=k|Y_2=0, Y_1=1)$ and $\theta_{3_k2_01}=\Pr(Y_3=k|Y_1=0,Y_2=0)=\Pr(Y_3|Y_1=2,Y_2=0)$.\label{fig:CSI}}
\end{figure}
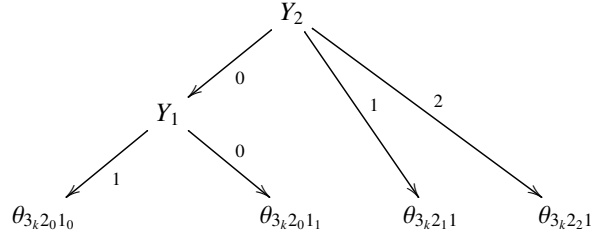
The last constraint cannot be represented by a CSI-tree and is usually referred to as a \textit{partial independence} \citep{Pensar2015}. In our polynomial approach, both partial and context-specific independences can be straightforwardly imposed in the interpolating polynomial representation of the model. In fact the interpolating polynomial for the model in this example corresponds to the polynomial in equation (\ref{eq:BNpol}) where the appropriate indeterminates are substituted with $\theta_{3_k2_11}$, $\theta_{3_k2_21}$ and $\theta_{21_0}$. This polynomial is again multilinear and homogeneous, just like for all context-specific BNs embellished with CSI-trees and partial independences.
\end{example}

\label{sec:asy}

We notice here that the interpolating polynomial of a multilinear model is not necessarily homogenous, as for example the one associated to certain chain event graph models, as shown in \cite{Gorgen2015a}.

\subsection{Non-multilinear models}
Having discussed multilinear models, we now introduce more general structures which are often encountered in dynamic settings. Although many more models have this property, for instance dynamic chain graphs \cite{Anacleto2013} and dynamic chain event graphs \cite{Barclay2015}, for the purposes of this paper we focus here on the most commonly used model class of \textit{dynamic Bayesian networks} (DBNs) \citep{Murphy2002}. In \cite{Gorgen2015a} we showed that the so-called non square-free chain event graph is also a non-multilinear model.
  
\subsubsection{Dynamic Bayesian networks}
 DBNs extend the BN framework to dynamic and stochastic domains. As often in practice, we consider only stationary, feed-forward DBNs respecting the first order Markov assumption with a finite horizon $T\in\mathbb{N}$, see e.g. \citep{Koller2001}. This assumes that probabilities do not vary when shifted on time (stationarity), that current states only depend on the previous time point (first-order Markov assumption) and that contemporaneous variable cannot directly affect each other (feed-forward). These DBNs can be simply described by an initial distribution over the first time point and a BN having as vertex set two generic time slices. Such latter BN is usually called 2-Time slice Bayesian Network (2-TBN).
Let $\{\bm{Y}(t)\}_{t\in[T]}=\{Y_i(t):i\in[m]\}_{t\in[T]}$ be a time series. 
\begin{definition}
\label{def:2-TBN}
A \emph{2-TBN} for a time series $\{\bm{Y}(t)\}_{t\in[T]}$ is a BN with DAG $\G$ such that $V(\G)=\{Y_i(t),Y_{i}(t+1):i\in[m]\}$ and its edge set is such that there are no edges $(Y_{i}(r),Y_{j}(r))$, $i,j\in[m]$, $r=t,t+1$, $t\in[T-1]$.   
\end{definition} 
    
\begin{definition}
A \emph{DBN} for a time series $\{\bm{Y}(t)\}_{t\in[T]}$ is a pair $(\G,\G')$, such that $\G$ is a BN with vertex set $V(\G)=\{Y_i(1):i\in[m]\}$, and $\G'$ is a 2-TBN such that its  vertex set $V(\G')$ is equal to $\{Y_i(t),Y_{i}(t+1):i\in[m]\}$.
\end{definition}

\begin{example}
\label{ex:DBN}
Consider the problem of Example \ref{ex:BN} and suppose the newborn can acquire the disease  once a year. Suppose further that the screening test and the vaccine are available for kids up to four years old. This scenario can be modeled by a DBN with time horizon $T=4$, where $Y_i(t)$, $i\in[3]$, $t\in[4]$, corresponds to the variable $Y_i$ of Example \ref{ex:BN} measured in the $t$-th year. Suppose that the probabilities of parents choosing the screening test and vaccination depend on whether or not the newborn acquired the disease in the previous year only. Furthermore, there is evidence that kids have a higher chance of contracting the disease if they were sick the previous year, whilst a lower chance if vaccination was chosen. This situation can be described by the DBN in Figure \ref{fig:DBN} where at time $t=1$ the correlation structure of the non-dynamic problem is assumed.

For a finite time horizon $T=4$, the interpolating polynomial has $2^8\cdot 3^4$ monomials each of degree $12$. To show that this polynomial is not multilinear, consider the event that the screening test is always positive, that the parents always decline vaccination and that the newborn gets mildly sick in her first three years of life, denoted as $\bm{y}_T=(Y_1(t)=1,Y_2(s)=1,Y_3(t)=0, t\in[4],s\in[3])$. Let the parameters for the first time slice be denoted as in Example \ref{ex:BN} and denote for $t=2,3,4$
\begin{align*}
\hat{\theta}_{1_i2_j}&=\Pr(Y_1(t)=i\ |\ Y_2(t-1)=j), \ \ i=0,1, \ \ j=0,1,2, \\
\hat{\theta}_{2_i2_j3_k}&=\Pr(Y_2(t)=i\ |\ Y_2(t-1)=j, Y_3(t-1)=k), \ \ k=0,1, \ \ i,j=0,1,2, \\
\hat{\theta}_{3_i2_j}&=\Pr(Y_3(t)=i\ |\ Y_2(t-1)=j), \ \ i=0,1, \ \ j=0,1,2.
\end{align*}
The interpolating polynomial for this event equals
\begin{equation}
c_{\textnormal{DBN}}(\theta,\bm{y}_{T})=\theta_{1_1}\theta_{2_11_1}\theta_{3_02_11_1}\hat{\theta}^3_{1_12_1}\hat{\theta}^3_{3_02_1}\left(\hat{\theta}_{2_12_13_0}^3+\hat{\theta}_{2_12_13_0}^2\hat{\theta}_{2_22_13_0}+\hat{\theta}_{2_12_13_0}^2\hat{\theta}_{2_02_13_0}\right),
\end{equation}
which has indeterminates of degree 3 and 2 and therefore is not multilinear.
\end{example}

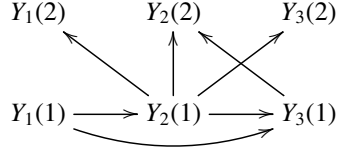
\begin{figure}
\centerline{
\xymatrix{
Y_1(2)&Y_2(2)&Y_3(2)\\
Y_1(1)\ar@/_1pc/[rr]\ar[r]&Y_2(1)\ar[r]\ar[u]\ar[ur]\ar[ul]&Y_3(1)\ar[lu]
}
}
\vspace{0.1cm}
\caption{A DBN having at time $t=1$ the DAG in Figure \ref{fig:BN} and a 2-TBN with an edge from $Y_2(t)$ to $Y_i(t+1)$, for $i\in[3]$, and the edge $(Y_3(t-1),Y_2(t))$.\label{fig:DBN}}
\end{figure}

Note that in the example above indeterminates can have degree up to to $T-1$, since this corresponds to the longest length of a path where the visited vertices can have probabilities that are identified in the \lq{u}nrolled\rq{} version of the DBN, i.e. one where the 2-TBN graph for time $t$ is recursively collated to the one of time $t-1$. From this observation the following follows.

\begin{proposition}
A DBN is a parametric model with monomial parametrisation, whose interpolating polynomial is homogeneous and each indeterminate can have degree lower or equal to $T-1$.
\end{proposition}

As for multilinear models, the interpolating polynomial of a non-multilinear model can be non-homogeneous. This is the case for example for certain non square-free chain event graphs.

\section{Divergence measures}
\label{sec:div}
In sensitivity analyses for discrete parametric statistical models we are often interested in studying how far apart from each other are two vectors of values of two probability mass functions $\bm{p}_{\theta}$ and $\bm{p}_{\tilde{\theta}}$ from the same model $\mathbb{P}_{\Psi}$. Divergence measures are used to quantify this dissimilarity between probability distributions. In this section we provide a brief introduction to these functions within the context of our discrete parametric probability models.

\begin{definition}
A \textit{divergence measure} $\mathcal{D}$ within a discrete parametric probability model $\mathbb{P}_{\Psi}$ is a function $\mathcal{D}(\cdot,\cdot):\mathbb{P}_{\Psi}\times\mathbb{P}_{\Psi}\rightarrow \mathbb{R}$ such that for all $\bm{p}_{\theta},\bm{p}_{\tilde{\theta}}\in\mathbb{P}_{\Psi}$:
\begin{itemize}
\item $\mathcal{D}(\bm{p}_{\theta},\bm{p}_{\tilde{\theta}})\geq 0$;
\item $\mathcal{D}(\bm{p}_{\theta},\bm{p}_{\tilde{\theta}})= 0$ iff $\bm{p}_{\theta}=\bm{p}_{\tilde{\theta}}$.
\end{itemize}
\end{definition}  
The larger the divergence between two probability mass functions $\bm{p}_{\theta}$ and $\bm{p}_{\tilde{\theta}}$, the more dissimilar these are. Notice that divergences are not formally metrics, since these do not have to be symmetric and respect the triangular inequality. We will refer to divergences with these two additional properties as \textit{distances}. 

The divergence most commonly used in practice is the KL divergence \cite{Kullback1951}.
\begin{definition}
The \textit{KL divergence} between $\bm{p}_{\theta},\bm{p}_{\tilde{\theta}}\in\mathbb{P}_{\Psi}$, $\mathcal{D}_{\textnormal{KL}}(\bm{p}_{\theta},\bm{p}_{\tilde{\theta}})$, is defined as
\begin{equation}
\mathcal{D}_{\textnormal{KL}}(\bm{p}_{\theta},\bm{p}_{\tilde{\theta}})=\sum_{\bm{y}\in\mathbb{Y}}p_{\theta}(\bm{y})\log\left(\frac{p_{\theta}(\bm{y})}{p_{\tilde{\theta}}(\bm{y})}\right),
\end{equation}
assuming $p_{\theta}(\bm{y}),p_{\tilde{\theta}}(\bm{y})>0$ for all $\bm{y}\in\mathbb{Y}$.
\end{definition} 
Notice that the KL divergence is not symmetric and thus $\mathcal{D}_{\textnormal{KL}}(\bm{p}_{\theta},\bm{p}_{\tilde{\theta}})\neq\mathcal{D}_{\textnormal{KL}}(\bm{p}_{\tilde{\theta}},\bm{p}_{\theta})$ in general. However both divergences can be shown to be a particular instance of a very general family of divergences, called $\phi$-divergences \cite{Ali1966,Csiszar1963}.
\begin{definition}
The $\phi$\textit{-divergence} between $\bm{p}_{\tilde{\theta}},\bm{p}_{\theta}\in\mathbb{P}_{\Psi}$, $\mathcal{D}_{\phi}(\bm{p}_{\tilde{\theta}},\bm{p}_{\theta})$, is defined as
\begin{equation}
\mathcal{D}_{\phi}(\bm{p}_{\tilde{\theta}},\bm{p}_{\theta})=\sum_{\bm{y}\in\mathbb{Y}}p_{\theta}(\bm{y})\phi\left(\frac{p_{\tilde{\theta}}(\bm{y})}{p_{\theta}(\bm{y})}\right), \hspace{0.5cm} \phi\in \Phi,
\end{equation}
where $\Phi$ is the class of convex functions $\phi(x)$, $x\geq 0$, such that $\phi(1)=0$, $0\phi(0/0)=0$  and $0\phi(x/0)=\lim_{x\rightarrow \infty}\phi(x)/x$. 
\end{definition}
So for example $\mathcal{D}_{\textnormal{KL}}(\bm{p}_{\theta},\bm{p}_{\tilde{\theta}})=\mathcal{D}_{\phi}(\bm{p}_{\tilde{\theta}},\bm{p}_{\theta})$ for $\phi(x)=-\log(x)$ and $\mathcal{D}_{\textnormal{KL}}(\bm{p}_{\tilde{\theta}},\bm{p}_{\theta})=\mathcal{D}_{\phi}(\bm{p}_{\tilde{\theta}},\bm{p}_{\theta})$ for $\phi(x)=x\log(x)$. Many other renowned divergences are in the family of $\phi$-divergences: for example $J$ divergences \cite{Jeffreys1946} and total variation distances (see \cite{Pardo2005} for a review).

The distance usually considered to study the dissimilarity of two probability mass functions in sensitivity analyses for discrete BNs is the aforementioned Chan-Darwiche distance. This distance is not a member of the $\phi$-divergence family.
\begin{definition}
The \textit{CD distance} between $\bm{p}_{\theta},\bm{p}_{\tilde{\theta}}\in\mathbb{P}_{\Psi}$, $\mathcal{D}_{\textnormal{CD}}(\bm{p}_{\theta},\bm{p}_{\tilde{\theta}})$, is defined as
\begin{equation}
\label{eq:CD}
\mathcal{D}_{\textnormal{CD}}(\bm{p}_\theta,\bm{p}_{\tilde{\theta}})=\log \max_{\bm{y}\in\mathbb{Y}}\frac{\bm{p}_{\tilde{\theta}}(\bm{y})}{\bm{p}_\theta(\bm{y})}-\log\min_{\bm{y}\in\mathbb{Y}} \frac{\bm{p}_{\tilde{\theta}}(\bm{y})}{\bm{p}_\theta(\bm{y})},
\end{equation}
where $0/0$ is defined as 1.
\end{definition}
 It has been noted that in sensitivity analysis in BNs, if one parameter of one CPT is varied, then the CD distance between the original and the varied BN equals the CD distance between the original and the varied CPT \cite{Chan2005}. This distributive property, and its associated computational simplicity, has lead to a wide use of the CD distance in sensitivity studies in discrete BNs.

\section{Sensitivity analysis in multilinear models}
\label{sec:multi}
We can now formalize sensitivity analysis techniques for multilinear parametric models. We focus on an extension of single full CPT analyses from BNs to generic multilinear models. Standard one-way sensitivity analyses can be seen as a special case of single full CPT analyses when only one parameter is allowed to be varied. We demonstrate in this section that all the results about one-way sensitivity analysis in BN models extend to single full CPT analyses in multilinear parametric models and therefore hold under much weaker assumptions about the structure of both the sample space and the underlying conditional independences. Before presenting these results we review the theory of \textit{covariation}. 
 
\subsection{Covariation}
\label{sec:covariation}
In one-way analyses one parameter within a parametrisation of a model is varied. When this is done, then \textit{some} of the remaining parameters need to be varied as well to respect the sum-to-one condition, so that the resulting measure is a probability measure. In the binary case this is straightforward, since the second parameter will be equal to one minus the other. But in generic discrete finite cases there are various considerations the user needs to take into account, as reviewed below.

Let $\theta_i\in\Theta$ be the parameter varied to $\tilde{\theta}_i$  and suppose this is associated to a random variable $Y_C$ in the random vector $\bm{Y}$. Let $\Theta_C=\{\theta_1,\dots,\theta_i,\dots,\theta_r\}\subseteq\Theta$ be the subset of the parameter set including $\theta_i$ describing the probability distribution of $Y_C$ and whose elements need to respect the sum to one condition. For instance $\Theta_C$ would include the entries of a CPT for a fixed combination of the parent variables in a BN model or the entries of a CPT associated to the conditional random variable from a leaf of a CSI-tree as in Figure \ref{fig:CSI}. Suppose further these parameters are indexed according to their values, i.e. $\theta_1\leq \cdots \leq \theta_i\leq\cdots \leq \theta_r$.  From \cite{Renooij2014} we then have the following definition.
\begin{definition}
\label{def:covar}
Let $\theta_i\in\Theta_C$ be varied to $\tilde{\theta}_i$. A \emph{covariation} scheme $\sigma(\theta_j,\tilde{\theta}_i):[0,1]^{2}\rightarrow[0,1]$ is a function that takes as input the value of both $\tilde{\theta}_i$ and $\theta_j\in\Theta_C$ and returns an updated value for $\theta_j$ denoted as $\tilde{\theta}_j$. 
\end{definition}

Different covariation schemes may entertain different properties which, depending on the domain of application, might be more or less desirable. We now list some of these properties from \cite{Renooij2014}.
\begin{definition}
\label{def:prop}
In the notation of Definition \ref{def:covar}, a covariation scheme $\sigma(\theta_j,\tilde{\theta}_i)$ is 
\begin{itemize}
\item \textit{valid}, if $\sum_{j\in[r]}\sigma(\theta_j,\tilde{\theta}_i)=1$;
\item \textit{impossibility preserving}, if for any parameter $\theta_j=0$, $j\neq i$, we have that $\sigma(\theta_j,\tilde{\theta}_i)=0$;
\item \textit{order preserving}, if $\sigma(\theta_1,\tilde{\theta}_i)\leq \cdots\leq \sigma(\theta_j,\tilde{\theta}_i)\leq \cdots\leq \sigma(\theta_r,\tilde{\theta}_i)$;
\item \textit{identity preserving}, if $\sigma(\theta_j,\theta_i)=\theta_j$, $\forall j\in[r]$;
\item \textit{linear}, if $\sigma(\theta_j,\tilde{\theta}_i)=\gamma_j\tilde{\theta}_i+\delta_j$, for $\gamma_{j}\in[0,1]$ and $\delta_j\in({-1},1)$.
\end{itemize}
\end{definition}
Of course any covariation scheme needs to be valid, otherwise the resulting measure is not a probability measure and any inference from the model would be misleading. Applying a linear scheme is very natural: if for instance $\delta_j=-\gamma_j$, then $\sigma(\theta_j,\tilde{\theta}_i)=\delta_j(1-\tilde{\theta}_i)$ and the scheme assigns a proportion $\delta_j$ of the remaining probability mass $1-\tilde{\theta}_i$ to the remaining parameters. Following \cite{Renooij2014} we now introduce a number of frequently applied  covariation schemes.
\begin{definition}
\label{def:schemes}
In the notation of Definition \ref{def:covar}, we define
\begin{itemize}
\item the \textit{proportional} covariation scheme, $\sigma_{\textnormal{pro}}(\theta_j,\tilde{\theta}_i)$, as
\[
\sigma_{\textnormal{pro}}(\theta_j,\tilde{\theta}_i)=\left\{
\begin{array}{ll}
\tilde{\theta}_i,& \mbox{if } j=i,\\
\frac{1-\tilde{\theta}_i}{1-\theta_i}\theta_j, &\mbox{otherwise}.
\end{array}
\right.
\] 
\item the \textit{uniform} covariation scheme, $\sigma_{\textnormal{uni}}(\theta_j,\tilde{\theta}_i)$, for $r=\#\Theta_C$, as
\[
\sigma_{\textnormal{uni}}(\theta_j,\tilde{\theta}_i)=\left\{
\begin{array}{ll}
\tilde{\theta}_i,& \mbox{if } j=i,\\
\frac{1-\tilde{\theta}_i}{r-1}, &\mbox{otherwise}.
\end{array}
\right.
\] 
\item the \textit{order preserving} covariation scheme, $\sigma_{\textnormal{ord}}(\theta_j,\tilde{\theta}_i)$, for $i\neq r$, as
\[
\sigma_{\textnormal{ord}}(\theta_j,\tilde{\theta}_i)=\left\{
\begin{array}{ll}
\tilde{\theta}_i,& \mbox{if } j=i,\\
\frac{\theta_j}{\theta_i}\tilde{\theta}_i, &\mbox{if } j<i \mbox{ and } \tilde{\theta}_i\leq \theta_i,\\
\frac{-\theta_j(1-\theta_{\textnormal{suc}})}{\theta_{\textnormal{suc}}\theta_i}\tilde{\theta}_i+\frac{\theta_j}{\theta_{\textnormal{suc}}}, &\mbox{if } j>i \mbox{ and } \tilde{\theta}_i\leq \theta_i,\\
\frac{\theta_j}{\theta_{\textnormal{max}}-\theta_i}(\theta_{\textnormal{max}} -\tilde{\theta}_i), &\mbox{if } j<i \mbox{ and } \tilde{\theta}_i> \theta_i,\\
\frac{\theta_j-\theta_{\textnormal{max}}}{\theta_{\textnormal{max}}-\theta_i}(\theta_{\textnormal{max}} -\tilde{\theta}_i)+\theta_{\textnormal{max}},&\mbox{if } j>i \mbox{ and } \tilde{\theta}_i> \theta_i,
\end{array}
\right.
\]
where $\theta_{\textnormal{max}}=1/(1+r-i)$ is the upper bound for $\tilde{\theta}_i$ and $\theta_{\textnormal{suc}}=\sum_{k=i+1}^r\theta_k$ is the original mass of the parameters succeeding $\theta_i$ in the ordering.
\end{itemize}
\end{definition}
Table \ref{table:prop} summarizes which of the properties introduced in Definition \ref{def:prop} the above schemes entertain (see \citep{Renooij2014} for more details). Under proportional covariation, to all the covarying parameters is assigned the same proportion of the remaining probability mass as these originally had. Although this scheme is not order preserving, it maintains the order among the covarying parameters. The uniform scheme on the other hand gives the same amount of the remaining mass to all covarying parameters. In addition, although the order preserving scheme is the only one that entertains the order preserving property, this limits the possible variations allowed. Note that this scheme is not only simply linear, but more precisely piece-wise linear, i.e. a function composed of straight-line sections. All the schemes in Definition \ref{def:schemes} are domain independent and therefore can be applied with no prior knowledge about the application of interest. Other schemes, for instance domain dependent or non-linear, have been defined, but these are not of interest for the theory we develop here.
 
\begin{table}
\begin{center}
\begin{tabular}{|l|c|c|c|c|c|}
\hline
\textbf{Scheme}/\textbf{Property}&valid&imp-pres&ord-pres&ident-pres&linear\\
\hline
Proportional & \cmark &\cmark &\xmark &\cmark &\cmark\\
\hline
Uniform & \cmark &\xmark &\xmark &\xmark &\cmark\\
\hline
Order Preserving & \cmark &\cmark &\cmark &\cmark &\cmark\\
\hline
\end{tabular}
\end{center}
\caption{Summary of the covariation schemes and the properties these entertain. \label{table:prop}}
\end{table}

\subsection{Sensitivity functions}
\label{sec:sensi}
We now generalize one-way sensitivity methods in BNs to the single full CPT case for general multilinear models. This type of analysis is simpler than other multi-way methods since the parameters varied/covaried never appear in the same monomial of the BN interpolating polynomial. So we now find an analogous CPT analysis in multilinear models which has the same property. Suppose we vary $n$ parameters $\theta_{1_i},\dots,\theta_{n_i}$ and denote by $\Theta_{j}=\{\theta_{j_1},\dots,\theta_{j_{r_j}}\}$, $j\in[n]$, the set of parameters including $\theta_{j_i}$ and associated to the same (conditional) random variable: thus respecting the sum to one condition. Assume these sets are such that $\cap_{j\in[n]}\Theta_{j}=\emptyset$. Note that a collection of such sets can not only be associated to the CPTs of one vertex given different parent configurations, but also, for instance, to the leaves of a CSI-tree as in Figure \ref{fig:CSI} or to the positions along the same \emph{cut} in a CEG \cite{Smith2008}.

We start by investigating sensitivity functions. These describe the effect of the variation of the parameters  $\theta_{1_i},\dots,\theta_{n_i}$ on the probability of an event $\mathbb{Y}_T\subseteq \mathbb{Y}$ of interest.  A sensitivity function $f_{\bm{y}_T}(\tilde{\theta}_{1_i},\dots,\tilde{\theta}_{n_i})$  equals the probability $\Pr(\bm{Y}\in\mathbb{Y}_T)\triangleq p_{\tilde{\theta}}(\bm{y}_T)$ and is a function in $\tilde{\theta}_{1_i},\dots, \tilde{\theta}_{n_i}$, where $\theta_{1_i},\dots, \theta_{n_i}$ are varied to  $\tilde{\theta}_{1_i},\dots, \tilde{\theta}_{n_i}$. Our parametric definition of a statistical model enables us to explicitly express these as functions of the covariation scheme for any multilinear model. Recall that $\mathbb{A}=\{\bm{\alpha}_{\bm{y}}\,|\, \bm{y}\in\mathbb{Y}\}$ and let $\T=\{\bm{\alpha}_{\bm{y}}\,|\,\bm{y}\in\mathbb{Y}_T\}$. Let $\mathbb{A}_{j},\bm{\mathbb{T}}_{j}\subset \{0,1\}^k$ be the subsets of $\mathbb{A}$ and $\bm{\mathbb{T}}$ respectively including the exponents where the entry associated to an indeterminate in $\Theta_{j}$ is not zero, $\mathbb{A}_{j_s}\subseteq\mathbb{A}_{j}$ and $\mathbb{T}_{j_s}\subseteq \mathbb{T}_j$  be the subsets including the exponents such that the entry relative to $\theta_{j_s}$ is not zero, $j\in[n], s\in[r_j]$. Formally
\[
\begin{array}{lcr}
\mathbb{A}_j=\{\bm{\alpha}_{\bm{y}}\,|\,\bm{y}\in\mathbb{Y},\alpha_{j_s,\bm{y}}\neq 0, s\in[r_j]\},&&\mathbb{T}_j=\{\bm{\alpha}_{\bm{y}}\,|\,\bm{y}\in\mathbb{Y}_{T},\alpha_{j_s,\bm{y}}\neq 0, s\in[r_j]\},\\
\mathbb{A}_{j_s}=\{\bm{\alpha}_{\bm{y}}\,|\,\bm{y}\in\mathbb{Y},\alpha_{j_s,\bm{y}}\neq 0\},&&\mathbb{T}_{j_s}=\{\bm{\alpha}_{\bm{y}}\,|\,\bm{y}\in\mathbb{Y}_{T},\alpha_{j_s,\bm{y}}\neq 0\}.
\end{array}
\]
Let $\mathbb{A}_{-j_s},\bm{\mathbb{T}}_{-j_s}\subseteq\{0,1\}^{k-1}$ be the sets including the elements in $\mathbb{A}_{j_s}$ and $\bm{\mathbb{T}}_{j_s}$, respectively, where the entry relative to $\theta_{j_s}\in\Theta_{j}$ is deleted. Lastly, let $\bm{\theta}_{-j_s}=\prod_{\theta_k\in\Theta\setminus\{\theta_{j_s}\}}\theta_k$.

\begin{proposition}
Consider a multilinear  model $\mathbb{P}_{\Psi}$ where the parameters $\theta_{j_i}\in\Theta_j$ are varied to $\tilde{\theta}_{j_i}$ and $\theta_{j_s}\in\Theta_j\setminus\{\theta_{j_i}\}$ is covaried according to a valid scheme $\sigma_j(\theta_{j_s},\tilde{\theta}_{j_i})$, $j\in[n]$, $s\in[r_j]\setminus\{j_i\}$. The sensitivity function $f_{\bm{y}_T}(\tilde{\theta}_{1_i},\dots,\tilde{\theta}_{n_i})$  can then be written as
\begin{equation}
f_{\bm{y_T}}(\tilde{\theta}_{1_i},\dots,\tilde{\theta}_{n_i})=\sum_{j\in[n]}\sum_{\bm{\alpha}\in\mathbb{T}_{-j_i}}\bm{\theta}_{-j_i}^{\bm{\alpha}}\tilde{\theta}_{j_i}+\sum_{j\in[n]}\sum_{s\in[r_j]\setminus\{j_i\}}\sum_{\bm{\alpha}\in\mathbb{T}_{-j_s}}\bm{\theta}_{-j_s}^{\bm{\alpha}}\sigma_j(\theta_{j_s},\tilde{\theta}_{j_i})+\hspace{-0.5cm}\sum_{\bm{\alpha}\in\mathbb{T}\setminus\cup_{k\in[n]}\mathbb{T}_k}\bm{\theta}^{\bm{\alpha}}.
\label{eq:sens}
\end{equation}
\label{prop:multi}
\end{proposition}
\begin{proof}
The probability of interest can be written as
\begin{align*}
p_\theta(\bm{y}_T)&=\sum_{\bm{\alpha}\in\mathbb{T}}\theta^{\bm{\alpha}}=\sum_{j\in[n]}\sum_{s\in[r_j]}\sum_{\bm{\alpha}\in\mathbb{T}_{-j_s}}\bm{\theta}_{-j_s}^{\bm{\alpha}}\theta_{j_s}+\sum_{\bm{\alpha}\in\mathbb{T}\setminus\cup_{k\in[n]}\mathbb{T}_k}\bm{\theta}^{\bm{\alpha}}\\
&=\sum_{j\in[n]}\sum_{\bm{\alpha}\in\mathbb{T}_{-j_i}}\bm{\theta}_{-j_i}^{\bm{\alpha}}\theta_{j_i}+\sum_{j\in[n]}\sum_{s\in[r_j]\setminus\{j_i\}}\sum_{\bm{\alpha}\in\mathbb{T}_{-j_s}}\bm{\theta}_{-j_s}^{\bm{\alpha}}\theta_{j_s}+\sum_{\bm{\alpha}\in\mathbb{T}\setminus\cup_{k\in[n]}\mathbb{T}_k}\bm{\theta}^{\bm{\alpha}}.
\end{align*}
The result follows by substituting the varying parameters with their varied version.
\end{proof}

From Proposition \ref{prop:multi} we can deduce that for a multilinear model, under a linear covariation scheme, the sensitivity function is multilinear.
\begin{corollary}
Under the conditions of Proposition \ref{prop:multi} and the linear covariation schemes $\sigma_j(\theta_{j_s},\tilde{\theta}_{j_i})=\gamma_{j_s}\tilde{\theta}_{j_i}+\delta_{j_s}$, the sensitivity function $f_{\bm{y}_T}(\tilde{\theta}_{1_i},\dots,\tilde{\theta}_{n_i})$  equals
\begin{equation}
\label{eq:linearsens}
f_{\bm{y}_T}(\tilde{\theta}_{1_i},\dots,\tilde{\theta}_{n_i})=\sum_{j\in[n]}a_j\tilde{\theta}_{j_i}+b,
\end{equation}
where
\begin{equation}
\label{now}
a_j=\sum_{\bm{\alpha}\in\T_{-j_i}}\bm{\theta}_{-j_i}^{\bm{\alpha}}+\sum_{s\in[r_j]\setminus \{j_i\}}\sum_{\bm{\alpha}\in\T_{-j_s}}\bm{\theta}_{-j_s}^{\bm{\alpha}}\gamma_{j_s},\hspace{0.5cm}
b=\sum_{j\in[n]}\sum_{s\in[r_j]\setminus\{j_i\}}\sum_{\bm{\alpha}\in \mathbb{T}_{-j_s}}\bm{\theta}_{-j_s}^{\bm{\alpha}}\delta_{j_s}+\hspace{-0.5cm}\sum_{\bm{\alpha}\in\mathbb{T}\setminus\cup_{k\in[n]}\mathbb{T}_k}\bm{\theta}^{\bm{\alpha}}.
\end{equation}
\label{cor1}
\end{corollary}
\begin{proof}
The result follows by substituting the definition of a linear covariation scheme into equation (\ref{eq:sens}) and then rearranging. 
\end{proof}

Therefore, under a linear covariation scheme, the sensitivity function is a multilinear function of the varying parameters $\tilde{\theta}_{j_i}$, $j\in[n]$. This was long known for BN models \cite{Castillo1997,Renooij2014,Gaag2007}. However, we have proven here that this feature is shared amongst all models having a multilinear interpolating polynomial. In BNs the computation of the coefficients $a_j$ and $b$ is particularly fast since for these models computationally efficient propagation techniques have been established. But these exist, albeit sometimes less efficiently, for other models as well (see e.g. \citep{Cowell2007} for chain graphs). Within our symbolic definition, we note however that once the exponent sets $\T_{-j_s}$, $s\in[r_j]$, are identified, then one can simply plug-in the values of the indeterminates to compute these coefficients. 

We now deduce the sensitivity function when parameters are varied using the popular proportional scheme.
\begin{corollary}
Under the conditions of Proposition \ref{prop:multi} and proportional covariation scheme $\sigma_{j}(\theta_{j_s},\tilde{\theta}_{j_i})=\frac{1-\tilde{\theta}_{j_i}}{1-\theta_{j_i}}\theta_{j_s}$, the sensitivity function, $f_{\bm{y}_T}(\tilde{\theta}_{1_i},\dots,\tilde{\theta}_{n_i})$  can be written in the multilinear form of equation (\ref{eq:linearsens}), where
\[
a_j=\sum_{\bm{\alpha}\in\T_{-j_i}}\bm{\theta}_{-j_i}^{\bm{\alpha}}-\sum_{s\in[r_j]\setminus \{j_i\}}\sum_{\bm{\alpha}\in\T_{j_s}}\frac{\bm{\theta}^{\bm{\alpha}}}{1-\theta_{j_i}},\hspace{0.5cm} b=\sum_{j\in[n]}\sum_{s\in[r_j]\setminus\{j_i\}}\sum_{\bm{\alpha}\in \mathbb{T}_{j_s}}\frac{\bm{\theta}^{\bm{\alpha}}}{1-\theta_{j_i}}+\hspace{-0.2cm}\sum_{\bm{\alpha}\in\mathbb{T}\setminus\cup_{k\in[n]}\mathbb{T}_k}\bm{\theta}^{\bm{\alpha}}.
\]
\end{corollary}
\begin{proof}
For a proportional scheme the coefficients in the definition of a linear scheme equals $\gamma_{j_s}=-\theta_{j_s}/(1-\theta_{j_i})$ and $\delta_{j_s}=\theta_{j_s}/(1-\theta_{j_i})$. By substituting these expressions into equation (\ref{now}) we have that 
\[
a_j=\sum_{\bm{\alpha}\in\T_{-j_i}}\bm{\theta}_{-j_i}^{\bm{\alpha}}-\sum_{s\in[r_j]\setminus \{j_i\}}\sum_{\bm{\alpha}\in\T_{-j_s}}\bm{\theta}_{-j_s}^{\bm{\alpha}}\frac{\theta_{j_s}}{1-\theta_{j_i}},\hspace{0.35cm}
b=\sum_{\substack{j\in[n] \\ s\in[r_j]\setminus\{j_i\}}}\sum_{\bm{\alpha}\in \mathbb{T}_{-j_s}}\bm{\theta}_{-j_s}^{\bm{\alpha}}\frac{\theta_{j_s}}{1-\theta_{j_i}}+\hspace{-0.2cm}\sum_{\bm{\alpha}\in\mathbb{T}\setminus\cup_{k\in[n]}\mathbb{T}_k}\bm{\theta}^{\bm{\alpha}}.
\]
By noting that $\sum_{\bm{\alpha}\in \mathbb{T}_{-j_s}}\bm{\theta}_{-j_s}^{\bm{\alpha}}\theta_{j_s}=\sum_{\bm{\alpha}\in \mathbb{T}_{j_s}}\bm{\theta}^{\bm{\alpha}}$ the result then follows.
\end{proof}

It is often of interest to investigate the posterior probability of a target event $(\bm{Y}\in\mathbb{Y}_T)$ given that an event $(\bm{Y}\in\mathbb{Y}_O)$ has been observed, $\mathbb{Y}_T,\mathbb{Y}_O\subseteq\mathbb{Y}$. This can be represented by the \textit{posterior} sensitivity function $f_{\bm{y}_T}^{\bm{y}
_O}(\tilde{\theta}_{1_i},\dots,\tilde{\theta}_{n_i})$ describing the probability $\Pr(\bm{Y}\in\mathbb{Y}_T\,|\,\bm{Y}\in\mathbb{Y}_O)$ as a function of the varying parameters $\tilde{\theta}_{1_i},\dots,\tilde{\theta}_{n_i}$.  
\begin{corollary}
\label{cor:senspost}
Under the conditions of Corollary \ref{cor1}, a posterior sensitivity function $f_{\bm{y}_T}^{\bm{y}_O}(\tilde{\theta}_{1_i},\dots,\tilde{\theta}_{n_i})$ can be written as the ratio
\begin{equation}
\label{eq:senspost}
f_{\bm{y}_T}^{\bm{y}_O}(\tilde{\theta}_{1_i},\dots,\tilde{\theta}_{n_i})=\frac{\sum_{j\in [n]}a_j \tilde{\theta}_{j_i}+b}{\sum_{j\in [n]}c_j\tilde{\theta}_{j_i}+d},
\end{equation}
where $a_j,c_j\in[0,1]$, $j\in[n]$,  and $b,d\in(-1,1)$.
\end{corollary}
\begin{proof}
The result follows from equation (\ref{eq:linearsens}) and by noting that $\Pr(\bm{Y}\in\mathbb{Y}_T\,|\,\bm{Y}\in\mathbb{Y}_O)=\Pr(\bm{Y}\in\left\{\mathbb{Y}_T\cap\mathbb{Y}_O\right\})/\Pr(\bm{Y}\in\mathbb{Y}_O)$. 
\end{proof}
The form of the coefficients in Corollary \ref{cor:senspost} can be deduced by simply adapting the notation of equation (\ref{eq:sens}) to the events $\Pr(\bm{Y}\in\left\{\mathbb{Y}_T\cap\mathbb{Y}_O\right\})$ and $\Pr(\bm{Y}\in\mathbb{Y}_O)$ for the numerator and the denominator, respectively, of equation (\ref{eq:senspost}). Sensitivity functions describing posterior probabilities in BNs have been proven to entertain the form in equation (\ref{eq:senspost}). Again, Corollary \ref{cor:senspost} shows that this is so for any model having a multilinear interpolating polynomial. 

\begin{table}
\begin{center}
\begin{tabular}{|cccc|}
\hline
$\theta_{1_1}=0.4$,& $\theta_{2_21_1}=0.3$, & $\theta_{2_11_1}=0.5$,& $\theta_{21_0}=0.3$,\\
 $\theta_{3_02_21}=0.1$, & $\theta_{3_02_11}=0.3$, & $\theta_{3_02_01_1}=0.7$, & $\theta_{3_02_01_0}=0.8$.\\
\hline
\end{tabular}
\end{center}
\caption{Probability specifications for Example \ref{ex:CSBN}. \label{table:prob}}
\end{table}

\begin{example}
\label{ex:sensi}
Suppose the context-specific model definition in Example \ref{ex:CSBN} is completed by the probability specifications in Table \ref{table:prob}. Suppose we are interested in the event that parents do not decide for vaccination. Figure \ref{fig:sensi} shows the sensitivity functions for this event when $\theta_{2_11_1}$ (on the x-axis) and $\theta_{21_0}$ (on the y-axis) are varied and the other covarying parameter are changed with different schemes. We can notice that for all schemes the functions are linear in their arguments and that more precisely for an order-preserving scheme the sensitivity function is piece-wise linear. Notice that whilst uniform and proportional covariation assigns similar, although different, probabilities to the event of interest, under order-preserving covariation  the probability of interest changes very differently from the other schemes: a property we have often observed in our investigations.
\begin{figure}
    \centering
    \begin{subfigure}[b]{0.32\textwidth}
        \includegraphics[width=\textwidth]{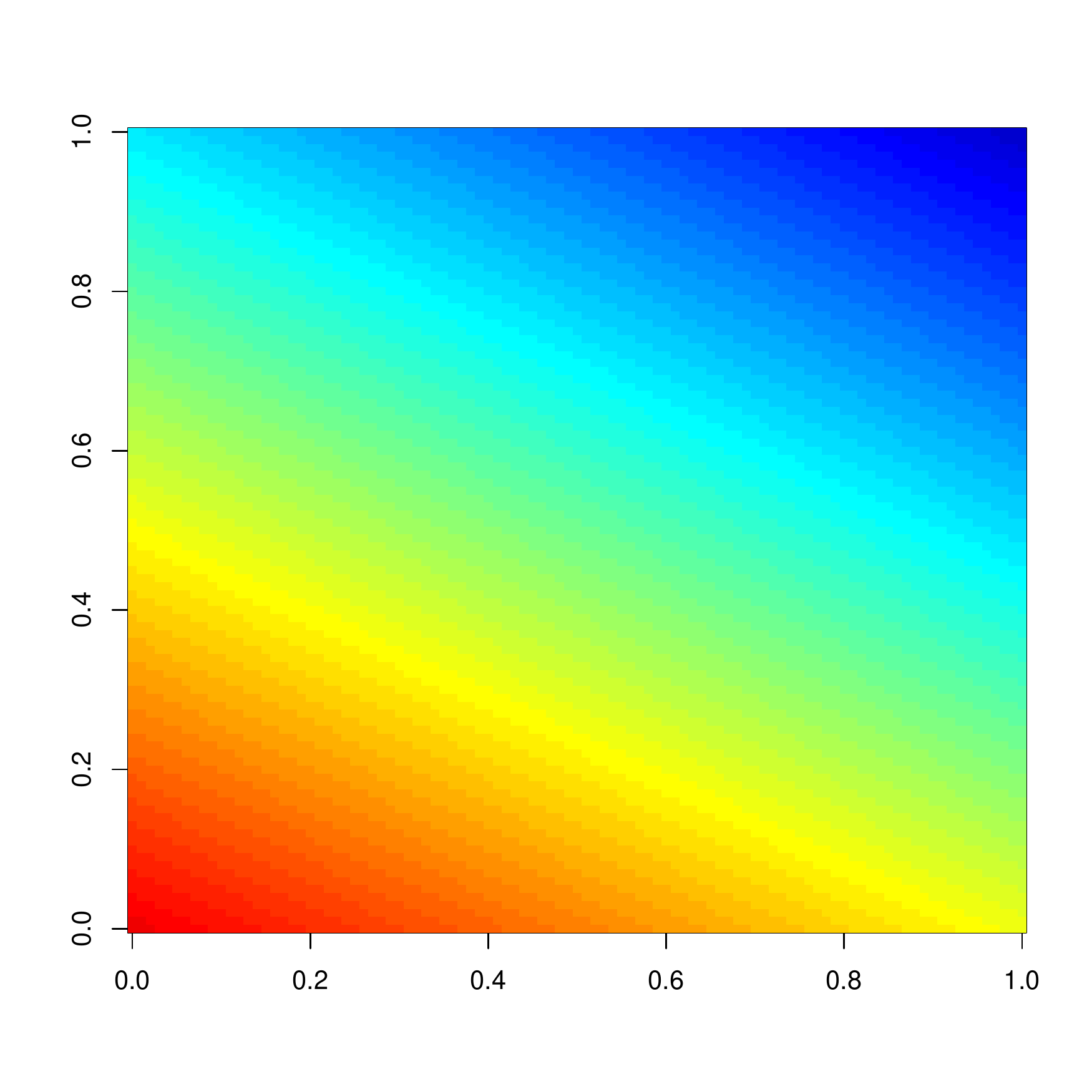}
        \caption{\tiny{Proportional.}}
    \end{subfigure}
    ~ 
    \begin{subfigure}[b]{0.32\textwidth}
        \includegraphics[width=\textwidth]{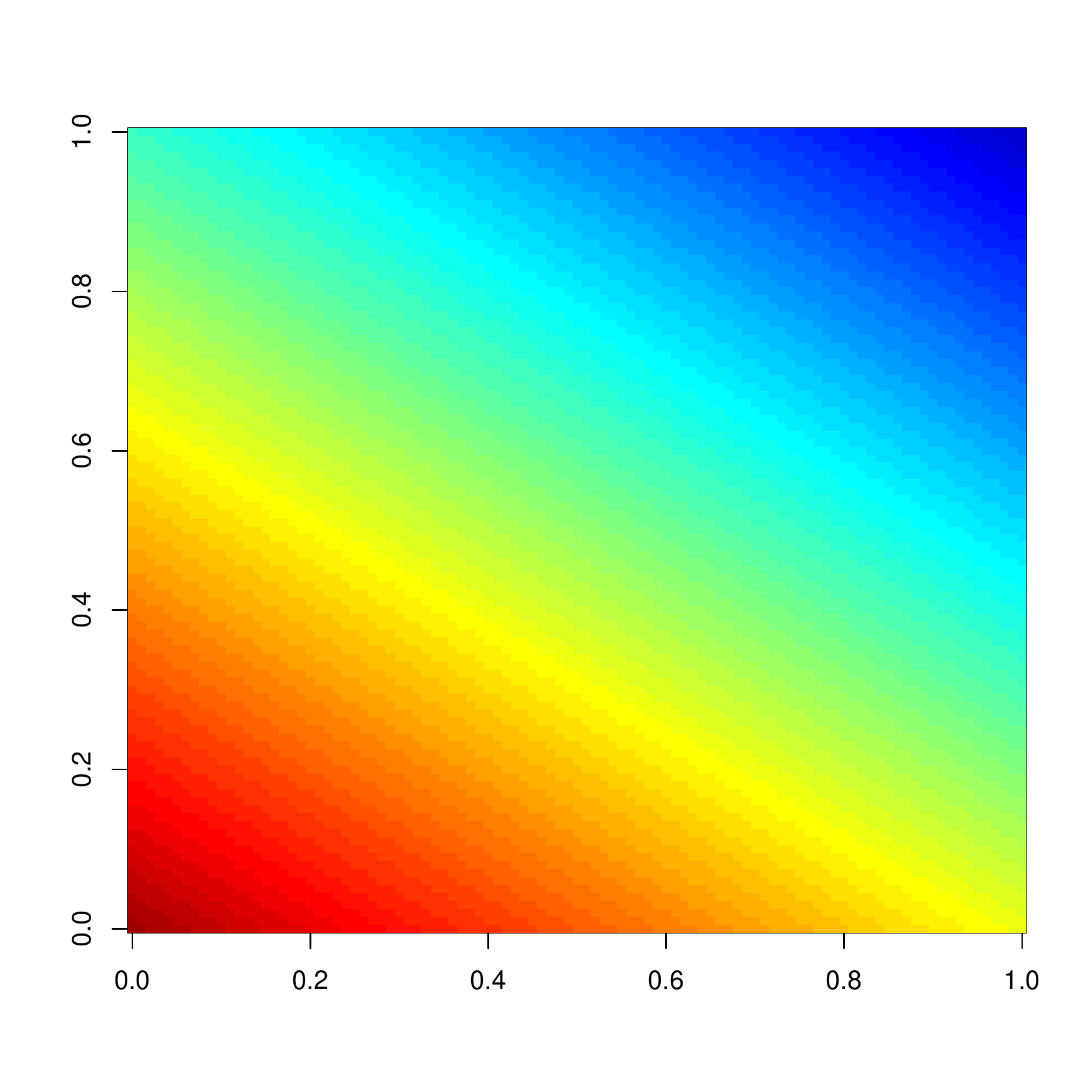}
        \caption{\tiny{Uniform.}}
    \end{subfigure}
     ~ 
    \begin{subfigure}[b]{0.32\textwidth}
        \includegraphics[width=\textwidth]{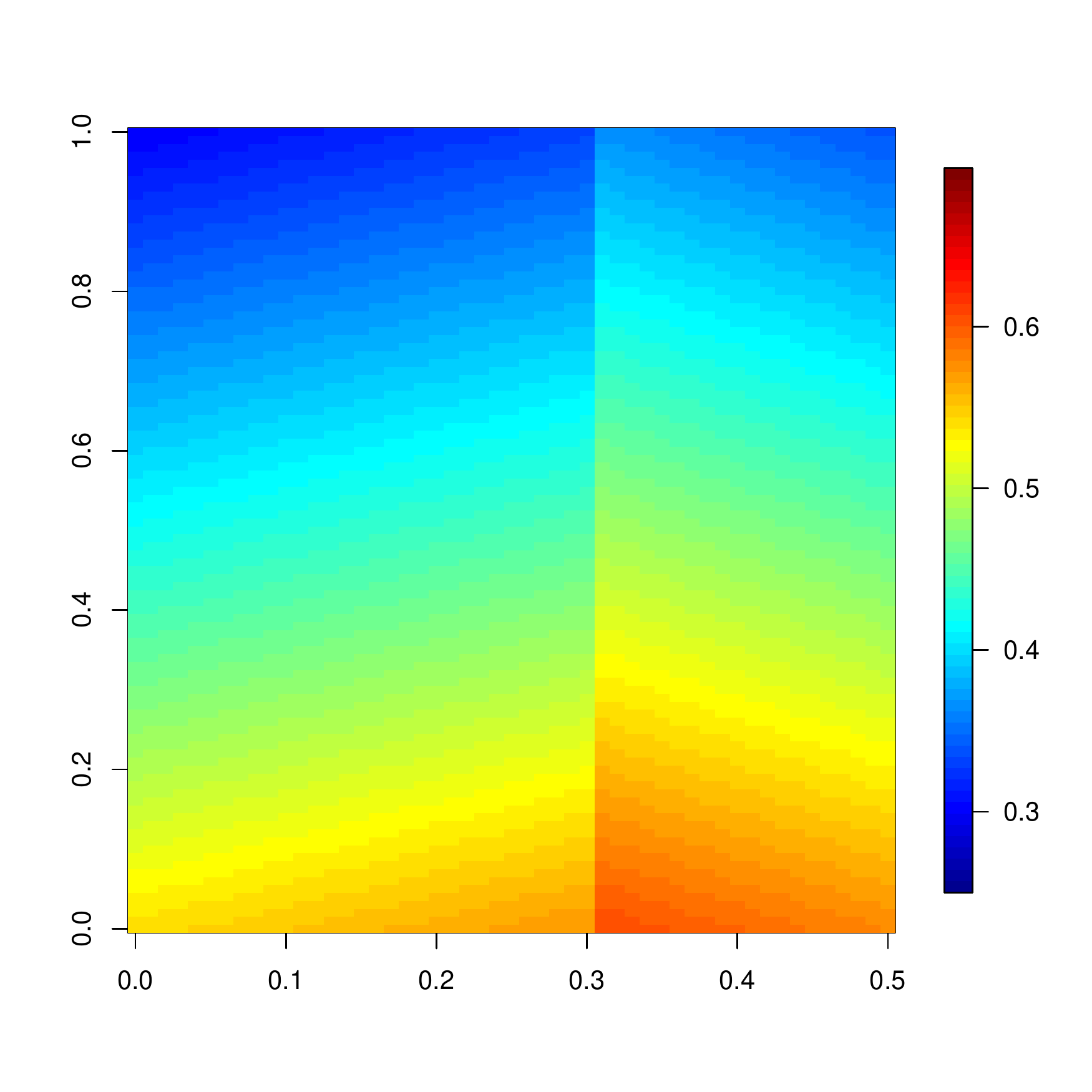}
        \caption{\tiny{Order-preserving.}}
    \end{subfigure}
    \caption{Sensitivity functions for Example \ref{ex:sensi} under different covariation schemes.\label{fig:sensi}}
    \end{figure}
\end{example}

\subsection{The Chan-Darwiche distance}
Whilst sensitivity functions study local changes, CD distances describe global variations in distributions \cite{Chan2005}. These can be used to study by how much two vectors of atomic probabilities vary in their distributional assumptions if one arises from the other via a covariation scheme. We are then interested in the global impact of that local change.

We next characterize the form of the CD distance for multilinear models in single full CPT analyses, first generalizing its form, again derived in \cite{Renooij2014} for BN models. We demonstrate that the distance depends only on the varied and covaried parameters: thus very easy to compute.

\begin{proposition}
\label{theo:CD}
Let $\bm{p}_\theta, \bm{p}_{\tilde{\theta}}\in\mathbb{P}_{\Psi}$, where $\mathbb{P}_{\Psi}$ is a multilinear parametric model and $\bm{p}_{\tilde{\theta}}$ arises from  $\bm{p}_\theta$ by varying $\theta_{j_i}$ to $\tilde{\theta}_{j_i}$ and $\theta_{j_s}\in\Theta_j\setminus\{\theta_{j_i}\}$ to $\tilde{\theta}_{j_s}=\sigma_j(\theta_{j_s},\tilde{\theta}_{j_i})$, where $\sigma_j(\theta_{j_s},\tilde{\theta}_{j_i})$ is a valid covariation scheme, $j\in[n], s\in[r_j]\setminus\{j_i\}$.  Then the CD distance between $\bm{p}_{\theta}$ and $\bm{p}_{\tilde{\theta}}$ is equal to
\begin{equation}
\label{eq:CDd}
\mathcal{D}_{\CD}(\bm{p}_\theta,\bm{p}_{\tilde{\theta}})=\log\max_{\substack{j\in[n] \\s\in[r_j]}}\frac{\tilde{\theta}_{j_s}}{\theta_{j_s}}-\log\min_{\substack{j\in[n] \\s\in[r_j]}}\frac{\tilde{\theta}_{j_s}}{\theta_{j_s}}.
\end{equation}
\end{proposition}

\begin{proof}
For a multilinear parametric model the CD distance can be written as 
\begin{align*}
\mathcal{D}_{\CD}(\bm{p}_\theta,\bm{p}_{\tilde{\theta}})&=\log\max_{\bm{\alpha}\in\mathbb{A}}\frac{\tilde{\bm{\theta}}^{\bm{\alpha}}}{\bm{\theta}^{\bm{\alpha}}}-\log\min_{\bm{\alpha}\in\mathbb{A}}\frac{\tilde{\bm{\theta}}^{\bm{\alpha}}}{\bm{\theta}^{\bm{\alpha}}}\\
&=\log\max\left\{\max_{\substack{j\in[n] \\ \bm{\alpha}\in\mathbb{A}_j}}\frac{\tilde{\bm{\theta}}^{\bm{\alpha}}}{\bm{\theta}^{\bm{\alpha}}},\max_{\bm{\alpha}\in\mathbb{A}\setminus\cup_{k\in[n]}\mathbb{A}_k}\frac{\tilde{\bm{\theta}}^{\bm{\alpha}}}{\bm{\theta}^{\bm{\alpha}}}\right\}-\log\min\left\{\min_{\substack{j\in[n] \\ \bm{\alpha}\in\mathbb{A}_j}}\frac{\tilde{\bm{\theta}}^{\bm{\alpha}}}{\bm{\theta}^{\bm{\alpha}}},\min_{\bm{\alpha}\in\mathbb{A}\setminus\cup_{k\in[n]}\mathbb{A}_k}\frac{\tilde{\bm{\theta}}^{\bm{\alpha}}}{\bm{\theta}^{\bm{\alpha}}}\right\}.
\end{align*}
If $\bm{\alpha}\in\mathbb{A}\setminus\cup_{k\in[n]}\mathbb{A}_k$, then $\tilde{\bm{\theta}}=\bm{\theta}$ and thus $\tilde{\bm{\theta}}^{\bm{\alpha}}/\bm{\theta}^{\bm{\alpha}}=1$. Because of the validity of the covariation scheme note that $\max_{\bm{\alpha}\in\mathbb{A}_j}\tilde{\bm{\theta}}^{\bm{\alpha}}/\bm{\theta}^{\bm{\alpha}}\geq 1$ and $\min_{\bm{\alpha}\in\mathbb{A}_j}\tilde{\bm{\theta}}^{\bm{\alpha}}/\bm{\theta}^{\bm{\alpha}}\leq 1$, for all $j\in[n]$. Thus
\[
\mathcal{D}_{\CD}(\bm{p}_\theta,\bm{p}_{\tilde{\theta}})=\log\max_{\substack{j\in[n] \\ \bm{\alpha}\in\mathbb{A}_j}}\frac{\tilde{\bm{\theta}}^{\bm{\alpha}}}{\bm{\theta}^{\bm{\alpha}}}-\log\min_{\substack{j\in[n] \\ \bm{\alpha}\in\mathbb{A}_j}}\frac{\tilde{\bm{\theta}}^{\bm{\alpha}}}{\bm{\theta}^{\bm{\alpha}}}.
\]
Now note that $\tilde{\bm{\theta}}=\bm{\theta}_{-j_s}\tilde{\theta}_{j_s}$, for a $\theta_{j_s}\in\Theta_j$, $j\in[n]$, since no two parameters in $\cup_{j\in[n]}\Theta_j$ can have exponent non-zero in the same monomial. Thus $\tilde{\bm{\theta}}^{\bm{\alpha}}/\bm{\theta}^{\bm{\alpha}}=\tilde{\theta}_{j_s}/\theta_{j_s}$ since $\bm{\alpha}\in\{0,1\}^k$ and the result follows.
\end{proof}
 
 We can now prove that the proportional covariation scheme is optimal for single full CPT analyses. This is important since a set of  parameters might be varied to change an uncalibrated probability of interest, but a user might want to achieve this by choosing a distribution as close as possible to the original one. Several authors have posed this problem for BNs without finding a definitive answer \citep{Chan2004,Renooij2014}. Here, exploiting our polynomial model representation, we can prove the optimality of the proportional scheme not only for BN models, but also for multilinear ones in single full CPT analyses.

\begin{theorem}
\label{theossimo}
Under the conditions of Proposition \ref{theo:CD} and proportional covariations $\sigma_j(\theta_{j_s},\theta_{j_i})$,  the CD distance between $\bm{p}_{\bm{\theta}}$ and $\bm{p}_{\tilde{\bm{\theta}}}$ is minimized and can be written in closed form as
\begin{equation}
\label{basta}
\mathcal{D}_{\CD}(\bm{p}_{\theta},\bm{p}_{\tilde{\theta}})=\log\max_{j\in[n]}\left\{\frac{\tilde{\theta}_{j_i}}{\theta_{j_i}},\frac{1-\tilde{\theta}_{j_i}}{1-\theta_{j_i}}\right\}-\log\min_{j\in[n]}\left\{\frac{\tilde{\theta}_{j_i}}{\theta_{j_i}},\frac{1-\tilde{\theta}_{j_i}}{1-\theta_{j_i}}\right\}.
\end{equation}
\end{theorem}

\begin{proof}
First note that we can write equation (\ref{eq:CDd}) as
\begin{equation}
\label{inproof1}
\mathcal{D}_{\CD}(\bm{p}_\theta,\bm{p}_{\tilde{\theta}})=
\log\max\left\{\max_{s\in[r_1]}\frac{\tilde{\theta}_{1_s}}{\theta_{1_s}},\dots,\max_{s\in[r_n]}\frac{\tilde{\theta}_{n_s}}{\theta_{n_s}}\right\}-\log\min\left\{\min_{s\in[r_1]}\frac{\tilde{\theta}_{1_s}}{\theta_{1_s}},\dots,\min_{s\in[r_s]}\frac{\tilde{\theta}_{n_s}}{\theta_{n_s}}\right\}.
\end{equation}
Now, let $\bar{\theta}_{j_i}=\tilde{\theta}_{j_i}$ and suppose $\bar{\theta}_{j_s}\in\Theta_{j}\setminus\{\theta_{j_i}\}$ is obtained via a valid covariation scheme, $j\in[n]$, $s\in[r_j]$. We want to prove that $\mathcal{D}_{\D}(\bm{p}_{\theta},\bm{p}_{\bar{\theta}})\geq \mathcal{D}_{\D}(\bm{p}_{\theta},\bm{p}_{\tilde{\theta}})$. Suppose now the proportional scheme is optimal for one-way sensitivity analyses. If this is true, we must have that, for all $j\in[n]$,
\[
\max_{s\in[r_j]}\frac{\bar{\theta}_{j_s}}{\theta_{j_s}}\geq \max_{s\in[r_j]}\frac{\tilde{\theta}_{j_s}}{\theta_{j_s}}, \hspace{1cm} \mbox{and}\hspace{1cm} \min_{s\in[r_j]}\frac{\bar{\theta}_{j_s}}{\theta_{j_s}}\leq \min_{s\in[r_j]}\frac{\tilde{\theta}_{j_s}}{\theta_{j_s}}.
\]
Therefore,
\[
\max\left\{\max_{s\in[r_1]}\frac{\bar{\theta}_{1_s}}{\theta_{1_s}},\dots,\max_{s\in[r_n]}\frac{\bar{\theta}_{n_s}}{\theta_{n_s}}\right\}\geq \max\left\{\max_{s\in[r_1]}\frac{\tilde{\theta}_{1_s}}{\theta_{1_s}},\dots,\max_{s\in[r_n]}\frac{\tilde{\theta}_{n_s}}{\theta_{n_s}}\right\},\]
and
\[
\min\left\{\min_{s\in[r_1]}\frac{\bar{\theta}_{1_s}}{\theta_{1_s}},\dots,\min_{s\in[r_n]}\frac{\bar{\theta}_{n_s}}{\theta_{n_s}}\right\}\leq \min\left\{\min_{s\in[r_1]}\frac{\tilde{\theta}_{1_s}}{\theta_{1_s}},\dots,\min_{s\in[r_n]}\frac{\tilde{\theta}_{n_s}}{\theta_{n_s}}\right\},
\]
from which the optimality condition follows.

We thus have to prove that for a single parameter change, the proportional covariation scheme minimizes the CD distance in any multilinear model. The proof follows similar steps to the one in \citep{Chan2002} for BNs. Fix $j\in[n]$ and note that if either $\theta_{j_i}=0$ or $\theta_{j_i}=1$ then the distance is infinite under both covariation schemes and the result holds. Consider now $\theta_{j_i}\in(0,1)$ and suppose $\bar{\theta}_{j_i}=\tilde{\theta}_{j_i}>\theta_{j_i}$. Under a proportional scheme, we have that 
\[
\max_{s\in[r_j]}\frac{\tilde{\theta}_{j_s}}{\theta_{j_s}}=\frac{\tilde{\theta}_{j_i}}{\theta_{j_i}} \,\,\,\, \mbox{ and }\,\,\,\, \min_{s\in[r_j]}\frac{\tilde{\theta}_{j_s}}{\theta_{j_s}}=\frac{\theta_{j_s}(1-\tilde{\theta}_{j_i})}{(\theta_{j_s}(1-\theta_{j_i}))}=\frac{(1-\tilde{\theta}_{j_i})}{(1-\theta_{j_i})}.
\]
Conversely, for the generic covariation scheme $\sigma(\theta_{j_s},\bar{\theta}_{j_i})$ we have that
\begin{align*}
\frac{1-\bar{\theta}_{j_i}}{1-\theta_{j_i}}&=\frac{\sum_{s\in[r_j]\setminus \{j_i\}}\bar{\theta}_{j_s}}{\sum_{s\in[r_j]\setminus \{j_i\}}\theta_{j_s}}=\frac{\sum_{s\in[r_j]\setminus \{j_i\}}\theta_{j_s}(\bar{\theta}_{j_s}/\theta_{j_s})}{\sum_{s\in[r_j]\setminus \{j_i\}}\theta_{j_s}}\\
&\geq \frac{\sum_{s\in[r_j]\setminus \{j_i\}}\theta_{j_s}(\min_{k\in[r_j]}\bar{\theta}_k/\theta_k)}{\sum_{s\in[r_j]\setminus \{j_i\}}\theta_{j_s}}=\min_{s\in[r_j]}\frac{\bar{\theta}_s}{\theta_s}.
\end{align*}
Thus since $(1-\bar{\theta}_{j_i})/(1-\theta_{j_i})=(1-\tilde{\theta}_{j_i})/(1-\theta_{j_i})$ we have that $\min_{s\in[r_j]}\tilde{\theta}_{j_s}/\theta_{j_s}\geq \min_{s\in[r_j]}\bar{\theta}_{j_s}/\theta_{j_s}$. Furthermore, 
\[\max_{s\in[r_j]}\frac{\bar{\theta}_{j_s}}{\theta_{j_s}}\geq \frac{\bar{\theta}_{j_i}}{\theta_{j_i}}=\frac{\tilde{\theta}_{j_i}}{\theta_{j_i}}=\max_{s\in[r_j]}\frac{\tilde{\theta}_{j_s}}{\theta_{j_s}}.
\]
 It then follows that $\mathcal{D}_{\CD}(\bm{p}_\theta,\bm{p}_{\bar{\theta}})\geq \mathcal{D}_{\CD}(\bm{p}_\theta,\bm{p}_{\tilde{\theta}})$ when $\tilde{\theta}_{j_i}>\theta_{j_i}$ for one-way analyses. For the case $\tilde{\theta}_{j_i}<\theta_{j_i}$ the proof mirrors the one presented here.  The explicit form of the distance under a proportional covariation schemes in equation (\ref{basta}) follows  by noting that the maximum and the minimum can either be $\tilde{\theta}_{j_i}/\theta_{j_i}$ or $(1-\tilde{\theta}_{j_i})/(1-\theta_{j_i})$. 
\end{proof}   

\begin{example}
\label{ex:CD}
\begin{figure}
    \centering
    \begin{subfigure}[b]{0.32\textwidth}
        \includegraphics[width=\textwidth]{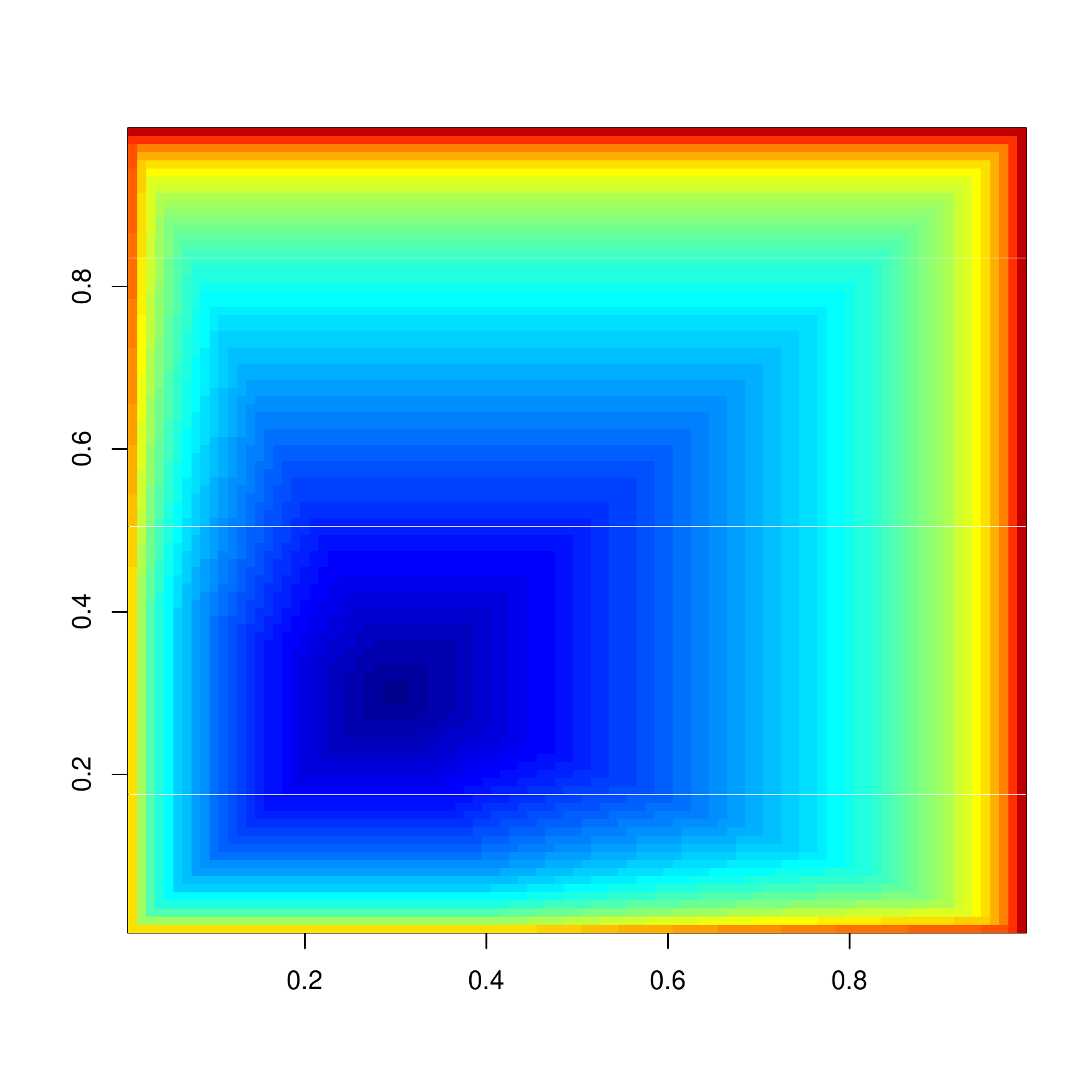}
        \caption{\tiny{Proportional.}\label{prop}}
    \end{subfigure}
    ~ 
    \begin{subfigure}[b]{0.32\textwidth}
        \includegraphics[width=\textwidth]{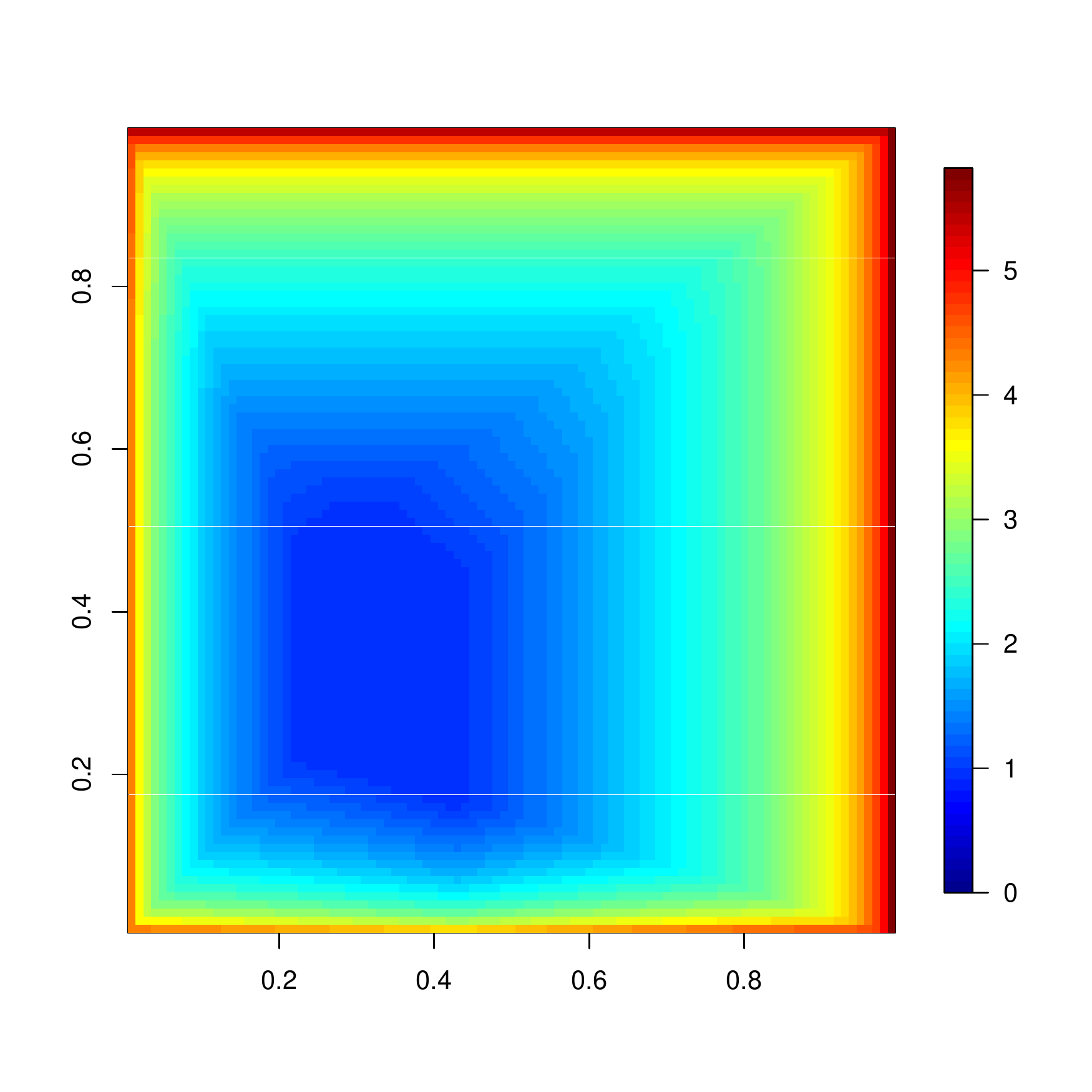}
        \caption{\tiny{Uniform.}\label{uni}}
    \end{subfigure}
     ~ 
    \begin{subfigure}[b]{0.32\textwidth}
        \includegraphics[width=\textwidth]{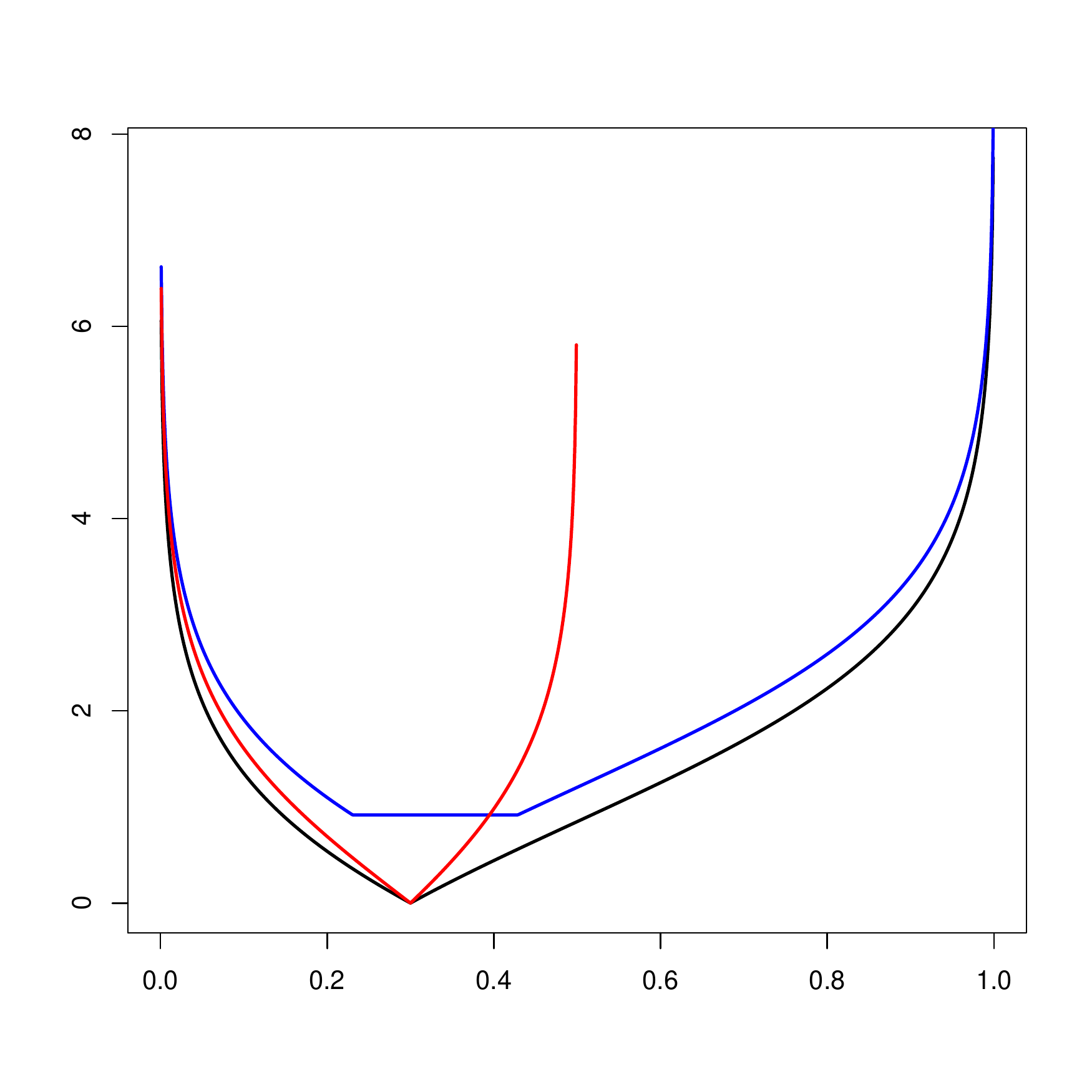}
        \caption{\tiny{One-way analysis.}\label{CD}}
    \end{subfigure}
    \caption{CD distances for Example \ref{ex:CD} under different covariation schemes: proportional (black), uniform (blue), order-preserving (red).\label{fig:CD}}
    \end{figure}
    
 In Figure \ref{fig:CD} we plott the CD distance between the varied and the original probability distributions for Example \ref{ex:CSBN} when $\theta_{2_11_1}$ (x-axis) and $\theta_{21_0}$ (y-axis in \ref{prop} and \ref{uni}) are varied for the covariation schemes so far considered. From Figures \ref{prop} and \ref{uni} we can get an intuition that the distance under proportional covariation is smaller than in the uniform case. This becomes clearer when we only let $\theta_{2_11_1}$ vary as shown in Figure \ref{CD} since the black line is always underneath the others.
\end{example}

\subsection{$\phi$-divergences}
Although the CD distance is widely used in sensitivity analyses, comparisons between two generic distributions are usually performed by computing the KL divergence. For one-way sensitivity analysis in BNs, the KL divergence equals the KL divergence between the original and varied conditional probability distribution of the manipulated parameter times the marginal probability of the conditioning parent configuration. This means that one way sensitivity analyses based on KL distances can become computationally infeasible, since this constant term might need to be computed an arbitrary large number of times. In Proposition \ref{prop:phi} below we demonstrate that this property is common to any $\phi$-divergence for any multilinear model and single full CPT analyses.
    
\begin{proposition}
\label{prop:phi}
Let $\bm{p}_\theta, \bm{p}_{\tilde{\theta}}\in\mathbb{P}_{\Psi}$, where $\mathbb{P}_{\Psi}$ is a multilinear parametric model and $\bm{p}_{\tilde{\theta}}$ arises from  $\bm{p}_\theta$ by varying $\theta_{j_i}$ to $\tilde{\theta}_{j_i}$ and $\theta_{j_s}\in\Theta_j\setminus\{\theta_{j_i}\}$ to $\tilde{\theta}_{j_s}=\sigma_j(\theta_{j_s},\tilde{\theta}_{j_i})$, where $\sigma_j(\theta_{j_s},\tilde{\theta}_{j_i})$ is a valid covariation scheme, $j\in[n], s\in[r_j]\setminus\{j_i\}$.  Then the $\phi$-divergence between $\bm{p}_{\tilde{\theta}}$ and $\bm{p}_{\theta}$ is equal to
\begin{equation}
\label{eq:phid}
\mathcal{D}_{\phi}(\bm{p}_{\tilde{\theta}},\bm{p}_{\theta})=\sum_{j\in[n]}\mathcal{D}_{\phi}(\bm{p}^j_{\tilde{\theta}},\bm{p}^j_{\theta})\sum_{s\in[r_j]}\sum_{\bm{\alpha}\in \mathbb{A}_{-j_s}}\bm{\theta}_{-j_s}^{\bm{\alpha}},
\end{equation}
where $\bm{p}^j_{\theta}$ denotes the vector of atomic probabilities in $\Theta_j$.
\end{proposition}

\begin{proof}
For a model with monomial parametrisation the $\phi$-divergence can be written as 
\begin{align*}
\mathcal{D}_{\phi}(\bm{p}_{\tilde{\theta}},\bm{p}_{\theta})&=\sum_{\bm{\alpha}\in \mathbb{A}}\bm{\theta}^{\bm{\alpha}}\phi\left(\frac{\tilde{\bm{\theta}}^{\bm{\alpha}}}{\bm{\theta}^{\bm{\alpha}}}\right)=\sum_{j\in[n]}\sum_{\bm{\alpha}\in\mathbb{A}_j}\bm{\theta}^{\bm{\alpha}}\phi\left(\frac{\tilde{\bm{\theta}}^{\bm{\alpha}}}{\bm{\theta}^{\bm{\alpha}}}\right)+\sum_{\bm{\alpha}\in\mathbb{A}\setminus\cup_{j\in[n]}\mathbb{A}_j}\bm{\theta}^{\bm{\alpha}}\phi\left(\frac{\tilde{\bm{\theta}}^{\bm{\alpha}}}{\bm{\theta}^{\bm{\alpha}}}\right).
\end{align*}
Notice that for $\bm{\alpha}\in\mathbb{A}\setminus\cup_{j\in[n]}\mathbb{A}_j$, $\tilde{\bm{\theta}}^{\bm{\alpha}}/\bm{\theta}^{\bm{\alpha}}=1$. Thus, since $\phi(1)=0$, we then have that
\begin{align*}
\mathcal{D}_{\phi}(\bm{p}_{\tilde{\theta}},\bm{p}_{\theta})&=\sum_{j\in[n]}\sum_{\bm{\alpha}\in\mathbb{A}_j}\bm{\theta}^{\bm{\alpha}}\phi\left(\frac{\tilde{\bm{\theta}}^{\bm{\alpha}}}{\bm{\theta}^{\bm{\alpha}}}\right)= \sum_{j\in[n]}\sum_{s\in[r_j]}\sum_{\bm{\alpha}\in\mathbb{A}_{-j_s}}\bm{\theta}_{-j_s}^{\bm{\alpha}}\theta_{j_s}\phi\left(\frac{\bm{\theta}_{-j_s}^{\bm{\alpha}}\tilde{\theta}_{j_s}}{\bm{\theta}_{-j_s}^{\bm{\alpha}}\theta_{j_s}}\right)\\
&=\sum_{j\in[n]}\sum_{s\in[r_j]}\theta_{j_s}\phi\left(\frac{\tilde{\theta}_{j_s}}{\theta_{j_s}}\right)\sum_{\bm{\alpha}\in\mathbb{A}_{-j_s}}\bm{\theta}_{-j_s}^{\bm{\alpha}}=\sum_{j\in[n]}\mathcal{D}_{\phi}(\bm{p}^j_{\tilde{\theta}},\bm{p}^j_{\theta})\sum_{s\in[r_j]}\sum_{\bm{\alpha}\in \mathbb{A}_{-j_s}}\bm{\theta}_{-j_s}^{\bm{\alpha}}.
\end{align*}
\end{proof}

The additional complexity of having to compute the constant term in equation (\ref{eq:phid}) has limited the use of KL divergences, and more generally $\phi$-divergences, in both practical and theoretical sensitivity investigations in discrete BNs. However, looking at probabilistic models from a polynomial point of view, we are able here to establish an additional strong theoretical justification for the use proportional covariation even in  single full CPT analyses, since this also minimizes any $\phi$-divergence.

\begin{theorem}
\label{theo:phi}
Under the conditions of Proposition \ref{prop:phi},  $\mathcal{D}_{\phi}(\bm{p}_{\tilde{\theta}},\bm{p}_{\theta})$ is minimized by the proportional covariation schemes $\sigma_j(\theta_{j_s},\tilde{\theta}_{j_i})=(1-\tilde{\theta}_{j_i})\theta_{j_s}/(1-\theta_{j_i})$.
\end{theorem}
\begin{proof}
Since $\sum_{s\in[r_j]}\sum_{\bm{\alpha}\in \mathbb{A}_{-j_s}}\bm{\theta}_{-j_s}^{\bm{\alpha}}$ in equation (\ref{eq:phid}) is a positive constant, $\mathcal{D}_{\phi}(\bm{p}_{\tilde{\theta}},\bm{p}_{\theta})$ is minimized if each $\mathcal{D}_{\phi}(\bm{p}^j_{\tilde{\theta}},\bm{p}^j_{\theta})$ attains its minimum. Fix a $j\in[n]$. We use the method of Lagrange multipliers to demonstrate that $\mathcal{D}_{\phi}(\bm{p}^j_{\tilde{\theta}},\bm{p}^j_{\theta})$ is minimized by proportional covariation, subject to the constraint that $\sum_{s\in[r_j]}\tilde{\theta}_{j_s}-1=0$. Define 
\[
L=\sum_{s\in[r_j]}\theta_{j_s}\phi\left(\frac{\tilde{\theta}_{j_s}}{\theta_{j_s}}\right)-\lambda\left(\sum_{s\in[r_j]}\tilde{\theta}_{j_s}-1\right).
\]
Taking the first derivative of $L$ with respect to $\tilde{\theta}_{j_s}$ and equating it to zero gives
\[
\frac{\partial}{\partial\tilde{\theta}_{j_s}}L=\phi'\left(\frac{\tilde{\theta}_{j_s}}{\theta_{j_s}}\right)=\lambda,
\]
where $\phi'$ denotes the derivative of $\phi$. By inverting we then deduce that
\begin{equation}
\label{eq:blabla}
\tilde{\theta}_{j_s}=\phi'(\lambda)^{-1}\theta_{j_s}. 
\end{equation}
Since equation (\ref{eq:blabla}) holds for every $s\in[r_j]\setminus\{j_i\}$ we have that 
\begin{equation}
\label{blablabla}
\sum_{s\in[r_j]\setminus\{j_i\}}\tilde{\theta}_{j_s}=\phi'(\lambda)^{-1}\sum_{s\in[r_j]\setminus\{j_i\}}\theta_{j_s}
\end{equation}
Now take the first partial derivative of $L$ with respect to $\lambda$ and equate it to zero. This gives
\begin{equation}
\label{bla}
\frac{\partial}{\partial\lambda}L=\sum_{s\in[r_j]}\tilde{\theta}_{j_s}=1 \ \ \ \ \Longrightarrow \ \ \ \  \sum_{s\in[r_j]\setminus\{j_i\}}\tilde{\theta}_{j_s}=1-\tilde{\theta}_{j_i}
\end{equation}
Plugging the right hand side of (\ref{bla}) into (\ref{blablabla}), we deduce that 
\begin{equation}
\label{ult}
\phi'(\lambda)^{-1}=\frac{1-\tilde{\theta}_{j_i}}{\sum_{s\in[r_j]\setminus\{j_i\}}\theta_{j_s}}= \frac{1-\tilde{\theta}_{j_i}}{1-\theta_{j_i}}.
\end{equation}
Thus, by plugging (\ref{ult}) into (\ref{eq:blabla}) we conclude that 
\[
\tilde{\theta}_{j_s}=\frac{1-\tilde{\theta}_{j_i}}{1-\theta_{j_i}}\theta_{j_s}.
\]
This is guaranteed to be a minimum by the convexity of the function $\phi$.
\end{proof}

\begin{example}
\label{ex:KL}
\begin{figure}
    \centering
    \begin{subfigure}[b]{0.32\textwidth}
        \includegraphics[width=\textwidth]{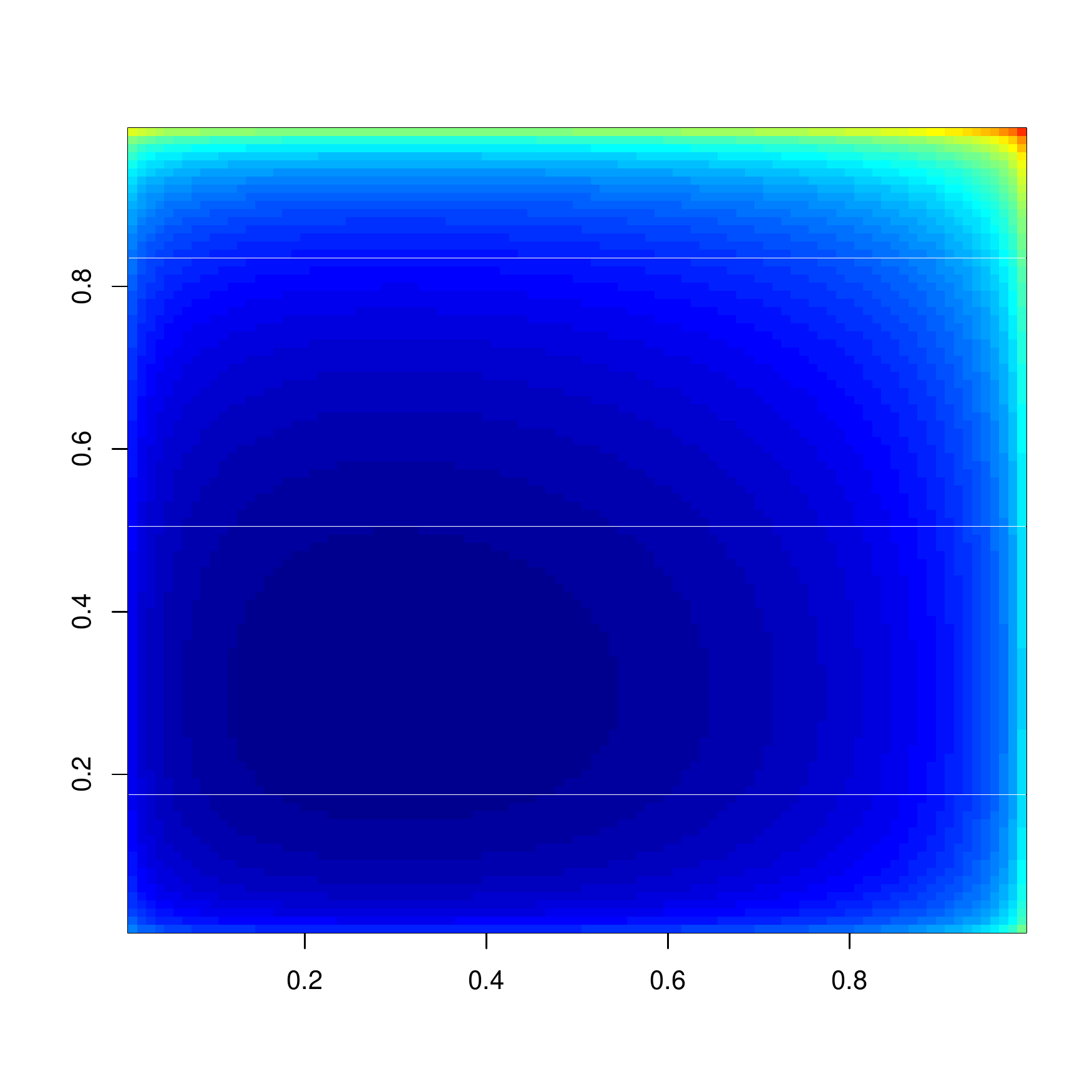}
        \caption{\tiny{Proportional.} \label{prop1}}
    \end{subfigure}
    ~ 
    \begin{subfigure}[b]{0.32\textwidth}
        \includegraphics[width=\textwidth]{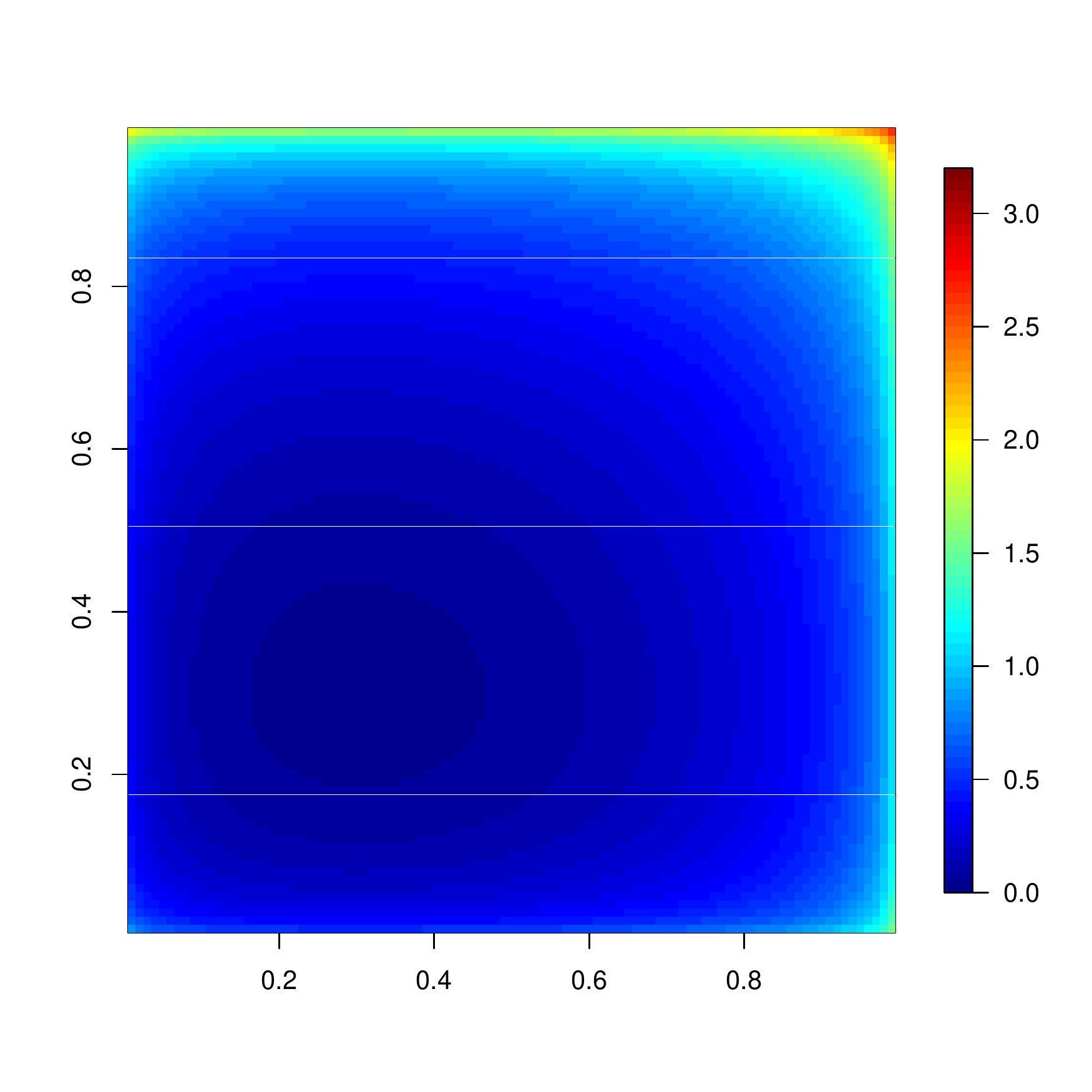}
        \caption{\tiny{Uniform.} \label{uni1}}
    \end{subfigure}
     ~ 
    \begin{subfigure}[b]{0.32\textwidth}
        \includegraphics[width=\textwidth]{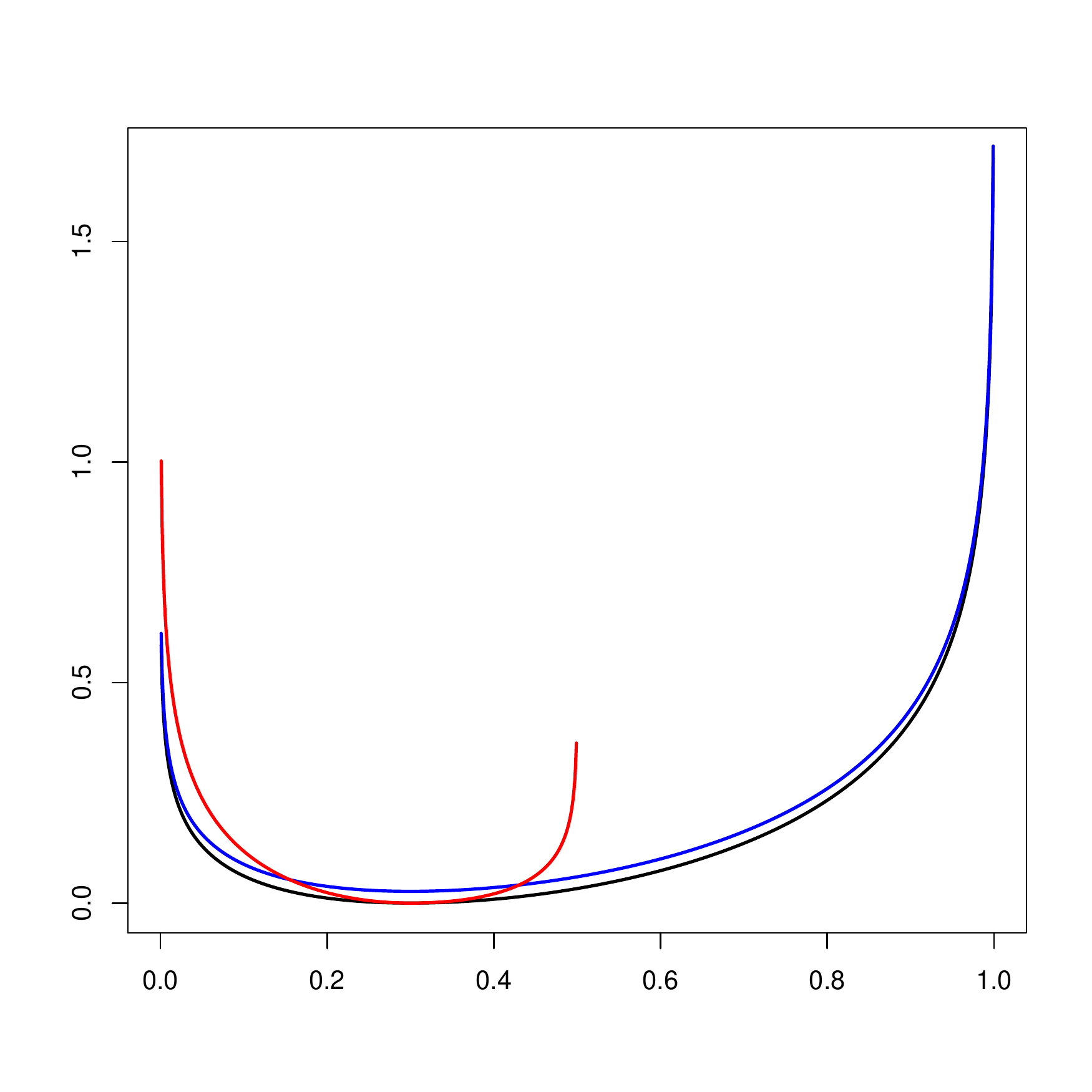}
        \caption{\tiny{One-way analysis.}\label{KL}}
    \end{subfigure}
    \caption{KL divergences for Example \ref{ex:KL} under different covariation schemes: proportional (black), uniform (blue), order-preserving (red).\label{fig:KL}}
    \end{figure}
    
  As an example of a $\phi$-divergence,  in Figure \ref{fig:KL} we plot KL$(\bm{p}_\theta,\bm{p}_{\tilde{\theta}})$ for Example \ref{ex:CSBN} when $\theta_{2_11_1}$ (x-axis) and $\theta_{21_0}$ (y-axis in \ref{prop1} and \ref{uni1}) are varied for the covariation schemes so far considered. From Figures \ref{prop1} and \ref{uni1} we can get an intuition that the KL divergence under proportional covariation is smaller than in the uniform case. This becomes clearer when we only let $\theta_{2_11_1}$ vary as shown in Figure \ref{KL} since the black line is always underneath the others.
\end{example}

\section{One-way sensitivity analysis in non-multilinear models}
\label{sec:pol}
For multilinear models we have been able to provide a unifying framework to perform various sensitivity analyses by deducing closed forms for both sensitivity functions and various divergences. Unfortunately, this is not possible for non-multilinear parametric models because these will depend on the degree of both the indeterminate to be varied and the covaried parameters, which is not necessarily equal to one. However, representing the model through its interpolating polynomial will enable us to study central properties of various sensitivity analyses. Since sensitivity functions and CD distances are most commonly applied, in this section we will only focus on these. Furthermore, because of the much more general structure underlying non-multilinear models, we will restrict our discussion to one-way sensitivity methods. As in Section \ref{sec:covariation}, we let $\theta_i\in\Theta_C$ be varied to $\tilde{\theta}_i$, where $\Theta_C=\{\theta_j: j\in[r]\}$ is the set of parameters including $\theta_i$ which need to respect the sum-to-one condition. 

\subsection{Sensitivity functions}
As in Section \ref{sec:multi} we let $f_{\bm{y}_T}$ denote a sensitivity function and  $f_{\bm{y}_T}^{\bm{y}_O}$ a posterior sensitivity function.

\begin{proposition}
\label{prop:polysens}
Consider a parametric model $\mathbb{P}_\Psi$ with monomial parametrisation $\Psi$. Let $\theta_{i}$ vary to $\tilde{\theta}_i$ and $\theta_j\in\Theta_C\setminus\{\theta_i\}$ covary according to a linear scheme. The sensitivity function $f_{\bm{y}_T}(\tilde{\theta}_i)$ is then a polynomial with degree lower or equal to $d$, where $d=\max_{\bm{\alpha}\in\mathbb{T}_C,j\in[r]}\{\alpha_j\}$. The posterior sensitivity function $f_{\bm{y}_T}^{\bm{y}_O}(\tilde{\theta}_i)$ is a rational function whose numerator and denominator are polynomials again with degree lower or equal to $d$.
\end{proposition}
\begin{proof}
This follows by noting that under the conditions of the proposition an exponent $\alpha_{i,\bm{y}}$ in Definition \ref{def:mon} of monomial parametrisation, for any $\bm{\alpha}_{\bm{y}}\in\mathbb{A}$ cannot be larger than $d$. 
\end{proof}

Similar results to Proposition \ref{prop:polysens} were presented in \citep{Charitos2006,Charitos2006a,Renooij2012} for some specific classes of DBNs only. 

The specific form of the sensitivity function in non-multilinear models depends both on the form of the interpolating polynomial, on the parameter to be varied and on the parameters that are consequently covaried. Therefore it is not possible to deduce a unique closed form expression since this will explicitly depend on the degree of all the above indeterminates. However, the interpolating polynomial enables us to identify a straightforward procedure to compute $f_{\bm{y}_T}(\tilde{\theta}_i)$ as follows:
\begin{enumerate}
\item determine the polynomial $c_{\mathbb{P}_\Psi}(\theta,\bm{y}_T)$ for an event $\mathbb{Y}_T$;
\item replace $\theta_i$ by $\tilde{\theta}_i$;
\item replace $\theta_{j}$ by $\sigma(\theta_{j},\tilde{\theta}_i)$ for $\theta_j\in\Theta_C\setminus\{\theta_i\}$.
\end{enumerate}

\begin{table}
\begin{center}
\scalebox{0.9}{
\begin{tabular}{|cccccc|}
\hline
$\hat{\theta}_{1_12_0}=0.4,$ & $\hat{\theta}_{1_12_1}=0.6,$ & $\hat{\theta}_{1_12_2}=0.7,$ &$\hat{\theta}_{3_02_0}=0.9,$&$\hat{\theta}_{3_02_1}=0.6,$&$\hat{\theta}_{3_02_2}=0.2,$\\
 $\hat{\theta}_{2_12_03_0}=0.3,$& $\hat{\theta}_{2_22_03_0}=0.2$& $\hat{\theta}_{2_12_13_0}=0.3$& $\hat{\theta}_{2_22_13_0}=0.5$& $\hat{\theta}_{2_12_23_0}=0.5$& $\hat{\theta}_{2_22_23_0}=0.3$.\\
\hline
\end{tabular}
}
\end{center}
\caption{Probability specifications for Example \ref{ex:sensDBN}. \label{table:prob1}}
\end{table}

\begin{figure}
\begin{center}
\includegraphics[scale=0.25]{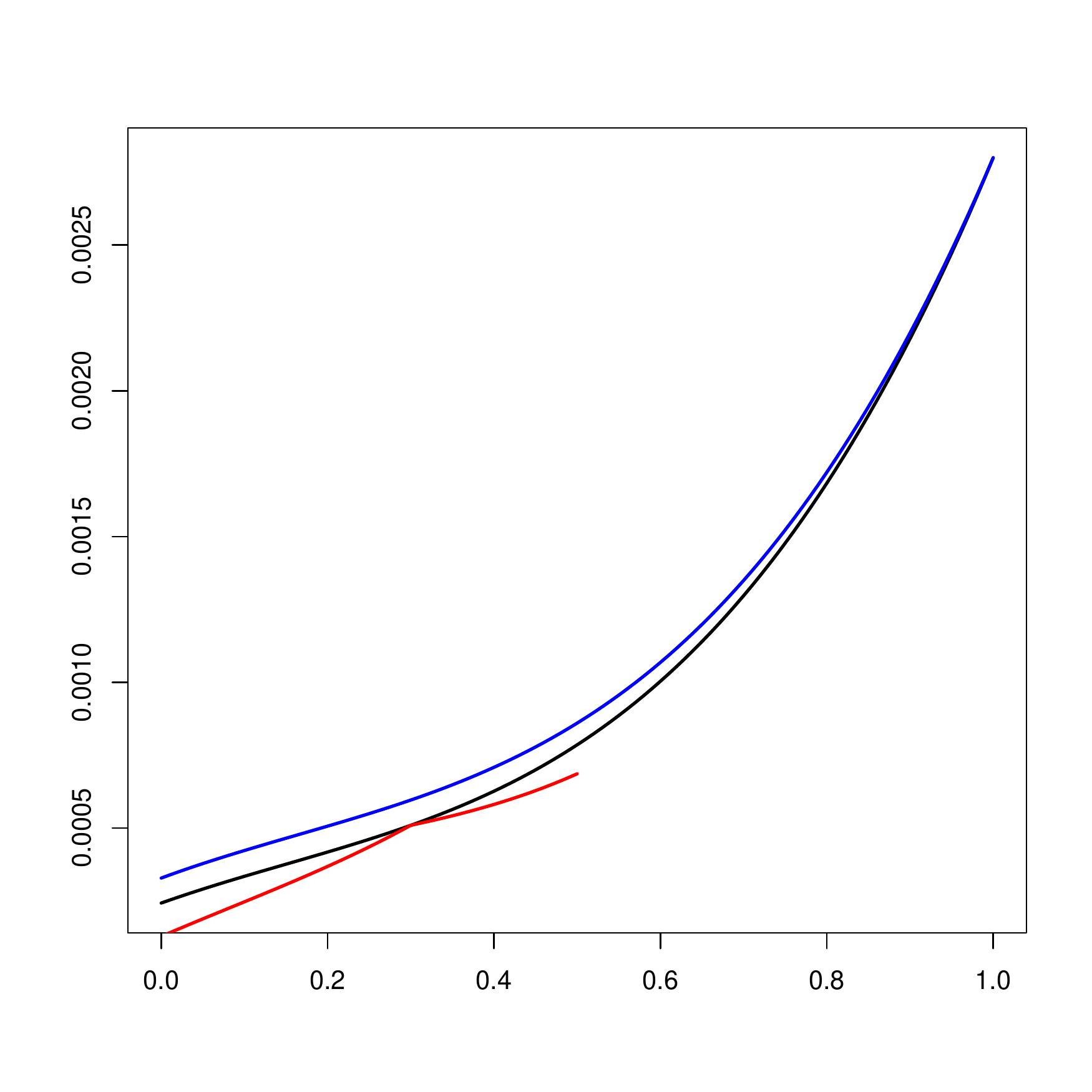}
\end{center}
\caption{Sensitivity functions in Example \ref{ex:sensDBN} when $\hat{\theta}_{2_12_13_0}$ (on the x-axis) is varied: proportional (black), uniform (blue) and order-preserving (red).\label{fig:sensDBN}}
\end{figure}

\begin{example}

\label{ex:sensDBN}

Suppose the definition of the DBN model in Example \ref{ex:DBN} is embellished by the probability specifications in Table \ref{table:prob1}. Suppose further the first time point distribution coincides with one defined in Table \ref{table:prop} for the context-specific BN in Example \ref{ex:CSBN}.  We are now interested in the event $(Y_1(t)=1, Y_2(1)=1, Y_2(4)=1 , Y_3(t)=0, t\in[4])$. By applying the above procedure to the interpolating polynomial of this DBN we can deduce that the sensitivity function when $\hat{\theta}_{2_12_13_0}$ is varied to $x $ equals
\[
ax^3+bx\sigma(\hat{\theta}_{2_02_13_0},x)+cx\sigma(\hat{\theta}_{2_22_13_0},x)+d\sigma(\hat{\theta}_{2_02_13_0},x)+e\sigma(\hat{\theta}_{2_22_13_0},x),
\] 
where $a,b,c,d,e\in[0,1]$. This is plotted in Figure \ref{fig:sensDBN} for different covariation schemes. As formalized in Proposition \ref{prop:polysens} these are not linear in their arguments, but more generally polynomial. As in the multilinear case, we can notice that the probability of interest under order-preserving covariation behaves rather differently than under  uniform and proportional covariation.  
\end{example}

\subsection{CD distance}
\label{why}
We next introduce a general procedure to compute the CD distance for parametric models with monomial parametrisation.
\begin{proposition}
\label{prop:CD}
Let $\bm{p}_\theta,\bm{p}_{\tilde{\theta}}\in\mathbb{P}_\Psi$, where $\mathbb{P}_\Psi$ is a parametric model with monomial parametrisation. The CD distance between $\bm{p}_\theta$ and $\bm{p}_{\tilde{\theta}}$ when $\theta_i$ is varied to $\tilde{\theta}_i$ and $\theta_j\in\Theta_C\setminus\{\theta_i\}$ is covaried to $\tilde{\theta}_j=\sigma(\theta_j,\tilde{\theta}_i)$, $j\in[r]\setminus\{i\}$, according to a valid covariation scheme can be computed from the interpolating polynomial as follows:
\begin{enumerate}
\item set $\bm{\theta}^{\bm{\alpha}}=0$ for all $\bm{\alpha}\in \mathbb{A}\setminus\mathbb{A}_C$;
\item set $\theta_k=1$, if $k\not\in[r]$, in any monomial $\bm{\theta}^{\bm{\alpha}}$, $\bm{\alpha}\in\mathbb{A}_C$;
\item call $\Phi$ the set of remaining monomials, $\bm{\phi}^{\bm{\alpha}}$ a generic element of $\Phi$ and $\tilde{\bm{\phi}}^{\bm{\alpha}}$ its varied version;
\item set $u=\max_{\Phi}\tilde{\bm{\phi}}^{\bm{\alpha}}/\bm{\phi}^{\bm{\alpha}}$ and $l=\min_{\Phi}\tilde{\bm{\phi}}^{\bm{\alpha}}/\bm{\phi}^{\bm{\alpha}}$;
\item compute $\CD(\bm{p}_\theta,\bm{p}_{\tilde{\theta}})=\log(u)-\log(l)$.
\end{enumerate}
It then follows that $\CD(\bm{p}_\theta,\bm{p}_{\tilde{\theta}})=\CD(\bm{p}_\phi,\bm{p}_{\tilde{\phi}})$, where $\bm{p}_\phi$ and $\bm{p}_{\tilde{\phi}}$ denote respectively the original and the varied probability mass function of the parametric model whose atomic probabilities are the elements of  $\Phi$. 
\end{proposition}
\begin{proof}
In a multilinear parametric model each atom is associated to a monomial $\bm{\theta}^{\bm{\alpha}}$. For all $\bm{\alpha}\in\mathbb{A}\setminus\mathbb{A}_C$, we have that 
\[
\frac{p_{\tilde{\theta}}(\bm{y})}{p_\theta(\bm{y})}=\frac{\tilde{\bm{\theta}}^{\bm{\alpha}}}{\bm{\theta}^{\alpha}}=1,
\]
 and there will always be ratios smaller or bigger than one because of the validity of the covariation scheme. Therefore, these monomials have no impact on the distance (step 1). For $\bm{\alpha}\in\mathbb{A}_C$, we have that 
\[
\frac{\tilde{\bm{\theta}}^{\bm{\alpha}}}{\bm{\theta}^{\bm{\alpha}}}=\frac{\tilde{\bm{\theta}}_C^{\bm{\alpha}_C}}{\bm{\theta}_C^{\bm{\alpha}_C}}\triangleq\frac{\tilde{\bm{\phi}}^{\bm{\alpha}}}{\bm{\phi}^{\bm{\alpha}}},
\]
where $\bm{\theta}_C=\prod_{j\in[r]}\theta_j$ and $\bm{\alpha}_C\in\mathbb{N}_0^r$ is the associated exponent vector. Therefore the distance depends only on the monomials computed in steps 2 and 3. The result then follows from the definition of the CD distance. 
\end{proof}

\begin{example}
\label{ex:CD1}
By following the procedure in Proposition \ref{prop:CD} we deduce that the set $\Phi$ for the DBN model in Example \ref{ex:DBN} when $\hat{\theta}_{2_12_13_0}$ is varied equals
\begin{multline*}
\Phi=\left\{\hat{\theta}_{2_12_13_0}^3,\hat{\theta}_{2_12_13_0}^2\hat{\theta}_{2_02_13_0},\hat{\theta}_{2_12_13_0}^2\hat{\theta}_{2_22_13_0}\hat{\theta}_{2_12_13_0}^2,\hat{\theta}_{2_02_13_0}^2,\hat{\theta}_{2_22_13_0}^2,\right.\\\left.\hat{\theta}_{2_12_13_0}\hat{\theta}_{2_02_13_0},\hat{\theta}_{2_12_13_0}\hat{\theta}_{2_22_13_0},\hat{\theta}_{2_02_13_0}\hat{\theta}_{2_22_13_0},\hat{\theta}_{2_12_13_0},\hat{\theta}_{2_02_13_0},\hat{\theta}_{2_22_13_0}\right\}
\end{multline*}
The algorithm selects the maximum ratio and the minimum ratio between any of these monomials and their non-varied versions, and then use these in the standard formula of the CD distance. In Figure \ref{fig:CDDBN} we plot the CD distance as a function of $\hat{\theta}_{2_12_13_0}$ under different covariation schemes. This shows that the distance is smallest for the proportional covariation scheme. In Theorem \ref{theossimo} we showed that this is the case for single full CPT sensitivity analyses in multilinear models. Unfortunately, this result does not hold in the non-multilinear case as shown in the following example.

\end{example}

\begin{figure}
\begin{center}
\includegraphics[scale=0.25]{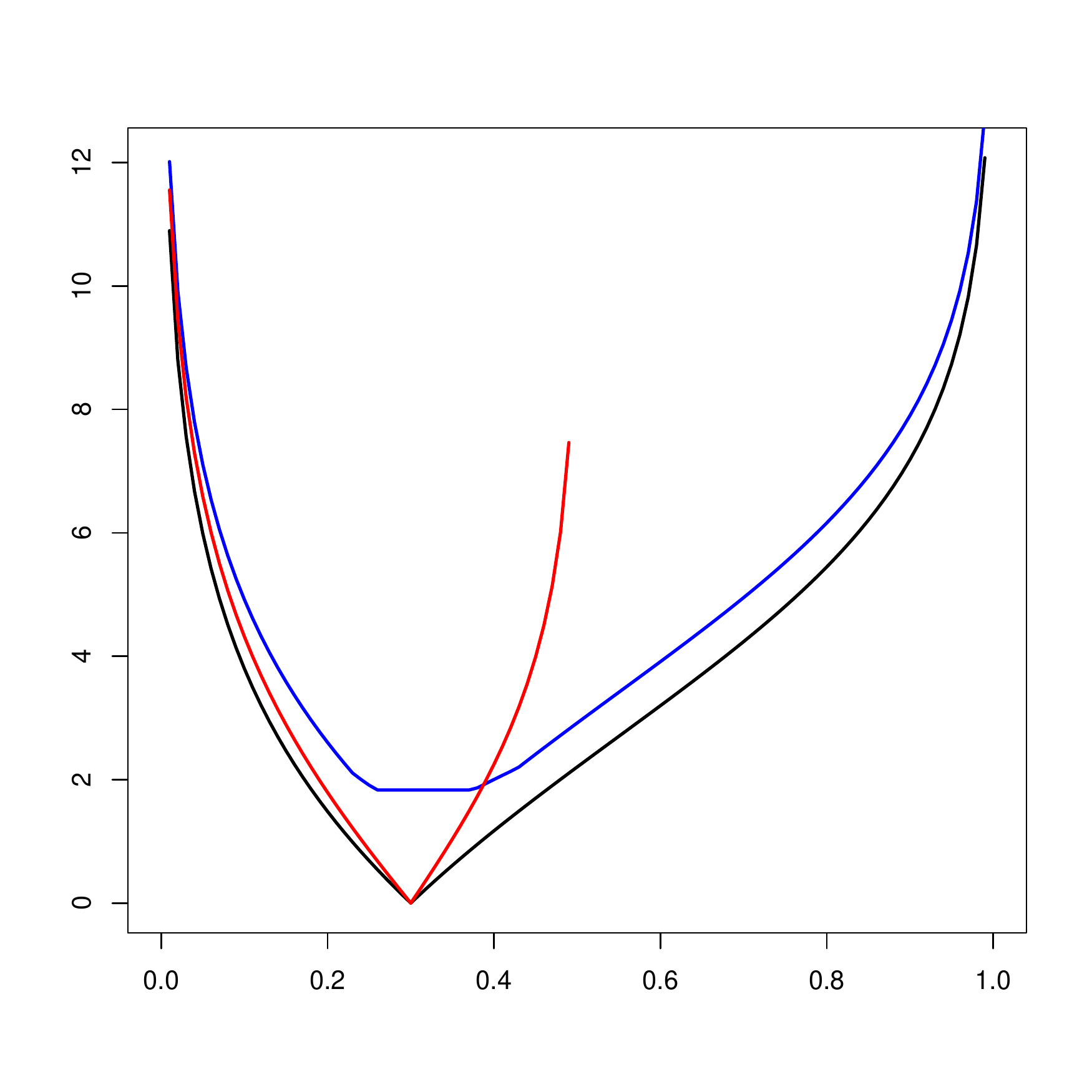}
\end{center}
\caption{CD distance in Example \ref{ex:CD1} when $\hat{\theta}_{2_12_13_0}$ (on the x-axis) is varied: proportional (black), uniform (blue) and order-preserving (red).\label{fig:CDDBN}}
\end{figure}

\begin{example}
\label{ex}
Consider two random variables $Y_1$ and $Y_2$ and suppose $\mathbb{Y}_1=\mathbb{Y}_{2}=[3]$. Suppose also \[
\theta_{1i}=\Pr(Y_1=i)=\Pr(Y_2=i\,|\,Y_1=j), \,\,\,\,\, i\in[3],j\in[2].
\]
 Suppose we let $\theta_{11}$ vary and $\theta_{12}$ and $\theta_{13}$ covary according to a valid scheme. The set $\Phi$ of Proposition \ref{prop:CD} is then equal $\{\theta_{11}^2,\theta_{11}\theta_{12},\theta_{13},\theta_{12}^2,\theta_{11}\theta_{13},\theta_{12}\theta_{13}\}$. Suppose $\theta_{11}=0.33$, $\theta_{12}=0.33$, $\theta_{13}=0.34$ and let $\theta_{11}$ be varied to $0.4$. In this situation the CD distance under a proportional scheme is $2.52$, whilst under a uniform scheme the distance equals $2.50$. For this parameter variation, the uniform scheme would then be preferred to a proportional one if a user wishes to minimize the CD distance. Conversely, if $\theta_{11}$ is set to $0.2$ the distance is smaller under the proportional scheme $(2.89)$ than under the uniform one $(2.92)$.
\end{example}
Therefore, whilst for multilinear models the choice of updating probabilities with a proportional covariation scheme can be justified in terms of some \lq{o}ptimality criterion\rq{} based on the minimization of the CD distance, for non-multilinear models the choice of the covariation scheme becomes critical. Our examples demonstrated that inference can be greatly affected by the chosen covariation scheme. With no theoretical justification to use one over another, any output from such a sensitivity analysis of a non-multilinear model will be ambiguous unless a convincing rationale for the choice of the covariation scheme can be found.

We next deduce a closed form expression for the CD distance under the proportional covariation scheme.
\begin{proposition}
\label{prop:CCD}
Under the condition of Proposition \ref{prop:CD} and assuming a proportional covariation scheme the CD distance is
\begin{equation}
\label{eq:CDD}
\CD(\bm{p}_\theta,\bm{p}_{\tilde{\theta}})=\log\max_{\bm{\alpha}\in\mathbb{A}_{-C}}\left(\frac{1-\tilde{\theta}_i}{1-\theta_i}\right)^{|\bm{\alpha}_{-i}|}\frac{\tilde{\theta_i}^{\alpha_i}}{{\theta_i}^{\alpha_i}}-\log\min_{\bm{\alpha}\in\mathbb{A}_{-C}}\left(\frac{1-\tilde{\theta}_i}{1-\theta_i}\right)^{|\bm{\alpha}_{-i}|}\frac{\tilde{\theta_i}^{\alpha_i}}{{\theta_i}^{\alpha_i}},
\end{equation}
where $\mathbb{A}_{-C}\in\mathbb{N}^r_0$ is the set including the exponents in $\mathbb{A}_C$ where the entries relative to indeterminates not in $\Theta_C$ are deleted, $\bm{\alpha}_{-i}\in\mathbb{N}^{r-1}_0$ is an exponent where the entry relative to $\theta_i$ is deleted and $|\bm{\alpha}_{-i}|$ is the sum of its entries. 
\end{proposition}
\begin{proof}
From Proposition \ref{prop:CD} we can write the CD distance as 
\[
\CD(\bm{p}_\theta,\bm{p}_{\tilde{\theta}})=\log\max_{\bm{\alpha}\in\mathbb{A}_{-C}}\frac{\tilde{\bm{\phi}}^{\bm{\alpha}}}{\bm{\phi}^{\bm{\alpha}}}-\log\min_{\bm{\alpha}\in\mathbb{A}_{-C}}\frac{\tilde{\bm{\phi}  }^{\bm{\alpha}}}{\bm{\phi}^{\bm{\alpha}}}.
\]
Recall that under a proportional scheme an indeterminate $\theta_j\in\Theta_C\setminus\{\theta_i\}$ is varied to $\tilde{\theta}_j=\theta_j(1-\tilde{\theta}_i)/(1-\theta_i)$.
We then have that, for any $\phi\in\Phi$,
\[
\tilde{\bm{\phi}}^{\bm{\alpha}}=\tilde{\theta}_i^{\alpha_i}\prod_{j\in[r]\setminus\{i\}}\left(\frac{1-\tilde{\theta}_i}{1-\theta_i}\theta_j\right)^{\alpha_j}
=\left(\frac{1-\tilde{\theta}_i}{1-\theta_i}\right)^{|\bm{\alpha}_{-i}|}\bm{\theta}_{-i}^{\bm{\alpha}_{-i}}\tilde{\theta}_i^{\alpha_i},
\]
and therefore
\[
\frac{\tilde{\bm{\phi}}^{\bm{\alpha}}}{\bm{\phi}^{\bm{\alpha}}}=\frac{\tilde{\theta}_i^{\alpha_i}}{\theta_i^{\alpha_i}}\left(\frac{1-\tilde{\theta}_i}{1-\theta_i}\right)^{|\bm{\alpha}_{-i}|}.
\]
Then,  since $\CD(\bm{p}_\theta,\bm{p}_{\tilde{\theta}})=\CD(\bm{p}_\phi,\bm{p}_{\tilde{\phi}})$ by Proposition \ref{prop:CD}, we deduce the closed form in equation (\ref{eq:CDD}) by substituting the above expression into the definition of CD distance. 
\end{proof}

\begin{example}
Consider again the monomial set $ \Phi$ of Example \ref{ex:CD1} and assume the parameters $\hat{\theta}_{2_02_13_0}$ and $\hat{\theta}_{2_22_13_0}$ are covaried according to a proportional scheme. Call $\hat{\theta}=\hat{\theta}_{2_12_13_0}$ and let $x$ be its varied version. From Proposition \ref{prop:CCD} we can deduce that the CD distance will then depend on the maximum and minimum value in the set of ratios
\[
\left\{\frac{x^3}{\hat{\theta}^3},\frac{x^2}{\hat{\theta}^2},\frac{x}{\hat{\theta}},\frac{x^2}{\hat{\theta}^2}\left(\frac{1-x}{1-\hat{\theta}}\right),\frac{x}{\hat{\theta}}\left(\frac{1-x}{1-\hat{\theta}}\right),\left(\frac{1-x}{1-\hat{\theta}}\right)^2,\frac{1-x}{1-\hat{\theta}}\right\}.
\]
\end{example}

\section{Discussion}
The definition of a parametric model by its interpolating polynomial has proven useful to investigate how changes in the input probabilities affect an output of interest. We have been able to demonstrate not only that standard results for one-way analyses in BN models are valid for a large number of other models and single full CPT investigations, but also new theoretical justifications for the use of proportional covariation based on a variety of divergence measures. Then, the flexibility of the interpolating polynomial representation enabled us to investigate an even larger class of models, for instance DBNs, extending sensitivity methods to dynamic settings. In this framework both sensitivity functions and CD distances exhibit different properties than in the simpler multilinear case, with the potential of even more informative sensitivity investigations. Importantly, we have been able to produce a new fast procedure to compute the CD distance in non-multilinear models.

Having demonstrated the usefulness of our polynomial approach in single full CPT analyses, we next plan to address the rather more complicated situation of generic multi-way analyses. In particular by representing probabilities in terms of monomials we can relate multi-way analyses in multilinear models to one-way sensitivities in non-multilinear ones. It can be seen that sensitivity functions for multi-way analyses will not simply be multilinear but also include interaction terms. Similarly, the CD distance will be affected by such interactions and not simply correspond to the CD distance of the appropriate CPT. Example \ref{ex} would therefore suggest that the proportional covariation scheme is not optimal in this context. However, we notice that the monomials from multi-way analyses in multilinear models are a subset of those arising from non-multilinear ones. Although these can be of an arbitrary degree, each indeterminate of the monomial will have exponent one by construction. Therefore, there is no conclusive proof of the non-optimality of the proportional scheme in generic multi-way analyses for multilinear models and BNs. However the polynomial representation of probabilities in BNs and related models gives us a promising starting point to start investigating this class of problems.

\section*{Acknowledgements}
M. Leonelli was supported by Capes, C. G\"{o}rgen was supported by the EPSRC grant EP/L505110/1 and J. Q. Smith was supported by the EPSRC grant EP/K039628/1.


\section*{References}
\bibliographystyle{plainnat} 
\bibliography{bib}

\end{document}